%% file: paper.tex
\documentclass[11pt]{article}
\pdfoutput=1

\input{math_commands.tex}

\usepackage{amsmath}
\usepackage{amsfonts}
\usepackage{amssymb}
\usepackage{amsthm}
\usepackage{url}
\usepackage{natbib}
\usepackage{MnSymbol}
\usepackage{tikz}
\usepackage[all]{xy}
\usepackage[arrowdel]{physics}
\usepackage{algorithm}
\usepackage{algpseudocode}
\usepackage{graphicx} 

\setcitestyle{square}

\usepackage[abbrvbib, nohyperref, preprint]{jmlr2e}

\ShortHeadings{Multiresolution Matrix Factorization and Wavelet Networks}{Hy and Kondor}
\firstpageno{1}

\begin{document}

\title{Learning Multiresolution Matrix Factorization and its Wavelet Networks on Graphs}

\author{\name Truong Son Hy \email hytruongson@uchicago.edu \\
       \addr Department of Computer Science \\
       University of Chicago \\
       Chicago, IL 60637, USA
       \AND
       \name Risi Kondor \email risi@uchicago.edu \\
       \addr Department of Computer Science and Department of Statistics \\
       University of Chicago \\
       Chicago, IL 60637, USA}

\editor{Kevin Murphy and Bernhard Sch{\"o}lkopf}

\maketitle

\input{Abstract}
\input{Introduction}
\input{Related_Work}
\input{Background_MMF}
\input{Stiefel_Manifold_Optimization}
\input{Reinforcement_Learning}

\input{Wavelet_Networks}
\input{Experiments}

\input{Software}
\input{Conclusion}




\vskip 0.2in
\bibliography{paper}

\clearpage

\appendix

\input{Notation}
\input{Multiresolution_Matrix_Factorization}

\input{Stiefel_Manifold_Optimization_2}
\input{Reinforcement_Learning_2}

\end{document}

%% file: math_commands.tex
\usepackage{amsmath,amsfonts,bm}

\def\eqref#1{equation~\ref{#1}}

\def\1{\bm{1}}

\def\vf{{\bm{f}}}
\def\vg{{\bm{g}}}

\def\vi{{\bm{i}}}
\def\vj{{\bm{j}}}

\def\vy{{\bm{y}}}

\def\mA{{\bm{A}}}
\def\mB{{\bm{B}}}
\def\mC{{\bm{C}}}
\def\mD{{\bm{D}}}

\def\mG{{\bm{G}}}
\def\mH{{\bm{H}}}
\def\mI{{\bm{I}}}

\def\mL{{\bm{L}}}
\def\mM{{\bm{M}}}

\def\mO{{\bm{O}}}
\def\mP{{\bm{P}}}

\def\mU{{\bm{U}}}

\def\mW{{\bm{W}}}
\def\mX{{\bm{X}}}
\def\mY{{\bm{Y}}}
\def\mZ{{\bm{Z}}}

\def\mLambda{{\bm{\Lambda}}}

\DeclareMathAlphabet{\mathsfit}{\encodingdefault}{\sfdefault}{m}{sl}
\SetMathAlphabet{\mathsfit}{bold}{\encodingdefault}{\sfdefault}{bx}{n}

\def\sA{{\mathbb{A}}}
\def\sB{{\mathbb{B}}}

\def\sH{{\mathbb{H}}}
\def\sI{{\mathbb{I}}}

\def\sO{{\mathbb{O}}}

\def\sS{{\mathbb{S}}}
\def\sT{{\mathbb{T}}}

\def\sV{{\mathbb{V}}}
\def\sW{{\mathbb{W}}}
\def\sX{{\mathbb{X}}}

\newcommand{\R}{\mathbb{R}}

\DeclareMathOperator*{\argmax}{arg\,max}

%% file: Abstract.tex
\begin{abstract}
Multiresolution Matrix Factorization (MMF) is unusual amongst fast matrix factorization algorithms in that 
it does not make a low rank assumption. This makes MMF especially well suited to modeling certain types of 
graphs with complex multiscale or hierarchical strucutre. While MMF promises to yields a useful wavelet basis, 
finding the factorization itself is hard, and existing greedy methods tend to be brittle. In this paper 
we propose a ``learnable'' version of MMF that carfully optimizes the factorization with a combination of 
reinforcement learning and Stiefel manifold optimization through backpropagating errors. We show that the 
resulting wavelet basis far outperforms prior MMF algorithms and provides the first version of this 
type of factorization that can be robustly deployed on standard learning tasks. 
\end{abstract}

\begin{keywords}
Matrix factorization, multiresolution analysis, manifold optimization, wavelet neural networks, graph learning
\end{keywords}

%% file: Introduction.tex
\section{Introduction} \label{sec:Introduction}

In certain machine learning problems large matrices have complex hierarchical structures that traditional 
linear algebra methods based on the low rank assumption struggle to capture. 
Multiresolution matrix factorization (MMF) is a relatively little used alternative paradigm
that is designed to capture structure at multiple different scales. 
MMF has been found to be particularly effective at compressing the adjacency or Laplacian 
matrices of graphs with complicated structure, such as social networks \citep{DBLP:conf/icml/KondorTG14}. \\ \\
MMF factorizations have a number of advantages, including the fact that they are easy to invert and have an 
interpretation as a form of wavelet analysis on the matrix and consequently on the underlying graph.  
The wavelets can be used e.g., for finding sparse approximations of graph signals. 
Finding the actual MMF factorization however is a hard optimization problem combining elements of 
continuous and combinatorial optimization. 
Most of the existing MMF algorithms just tackle this with a variety of greedy heuristics and are consequently 
brittle: the resulting factorizations typically have large variance and most of the time yield factorizations 
that are far from the optimal \citep{pmlr-v51-teneva16,8099564,NIPS2017_850af92f}. \\ \\
The present paper proposes an alternative paradigm to MMF optimization based on ideas from deep learning.  
Specifically, we employ an iterative approach to optimizing the factorization based on backpropagating the factorization 
error and a reinforcement learning strategy for solving the combinatorial part of the problem. 
While more expensive than the greedy approaches, we find that the resulting ``learnable'' MMF produces 
much better quality factorizations and a wavelet basis that is smoother and better reflects the structure 
of the underlying matrix or graph. 
Unsurprisingly, this also means that the factorization performs better in downstream tasks. \\ \\
To apply our learnable MMF algorithm to standard benchmark tasks, we also propose a wavelet extension of 
the Spectral Graph Networks algorithm of \citep{ae482107de73461787258f805cf8f4ed} which we call the 
Wavelet Neural Network (WNN). Our experiments show that the combination of learnable MMF optimization 
with WNNs achieves state of the art results on several graph learning tasks. Beyond just benchmark performance, 
the greatly improved stability of MMF optimization process and the similarity of the hierarchical 
structure of the factorization to the architecture of deep neural networks opens up the possibility 
of MMF being tightly integrated with other learning algorithms in the future. 

%% file: Related_Work.tex
\section{Related work} \label{sec:Related-work}

Compressing and estimating large matrices has been extensively studied from various directions, 
including 
\citep{Drineas2006FastMC},  
\citep{doi:10.1137/090771806}, 
\citep{NIPS2000_19de10ad} 
\citep{JMLR:v13:kumar12a}, 
\citep{10.1561/2200000035}, 
\citep{pmlr-v9-jenatton10a}. 
Many of these methods come with explicit guarantees but typically make the assumption that 
the matrix to be approximated is low rank. \\ \\
MMF is more closely related to other works on constructing wavelet bases on discrete spaces, 
including wavelets defined based on diagonalizing the diffusion operator or the normalized graph Laplacian 
\citep{COIFMAN200653} \citep{HAMMOND2011129} and multiresolution on trees 
\citep{10.5555/3104322.3104370} \citep{10.1214/07-AOAS137}. 
MMF has been used for matrix compression \citep{pmlr-v51-teneva16}, 
kernel approximation \citep{NIPS2017_850af92f} 
and inferring semantic relationships in medical imaging data \citep{8099564}. \\ \\
Most of the combinatorial optimization problems over graphs are NP-Hard, which means that no polynomial time solution can be developed for them. Many traditional algorithms for solving such problems involve using suboptimal heuristics designed by domain experts, and only produce approximations that are guaranteed to be some factor worse than the true optimal solution. Reinforcement learning (RL) proposes an alternative to replace these heuristics and approximation algorithms by training an agent in a supervised or self-supervised manner \citep{DBLP:journals/corr/BelloPLNB16} \citep{Mazyavkina21}. \citep{NIPS2017_d9896106} proposed the use of graph embedding network as the agent to capture the current state of the solution and determine the next action. Similarly, our learning algorithm addresses the combinatorial part of the MMF problem by gradient-policy algorithm that trains graph neural networks as the RL agent. \\ \\
Graph neural networks (GNNs) utilizing the generalization of convolution concept to graphs have been popularly applied to many learning tasks such as estimating quantum chemical computation, and modeling physical systems, etc. 
Spectral methods such as \citep{ae482107de73461787258f805cf8f4ed} provide one way to define convolution on 
graphs is via convolution theorem and graph Fourier transform (GFT). 
To address the high computational cost of GFT, \citep{xu2018graph} proposed to use the diffusion wavelet bases 
as previously defined by \citep{COIFMAN200653} instead for a faster transformation. 

%% file: Background_MMF.tex
\section{Background on Multiresolution Matrix Factorization}


The \textit{Multiresolution Matrix Factorization} (MMF) of a matrix $\mA\in\mathbb{R}^{n\times n}$ is a factorization of the form
$$\mA = \mU_1^T \mU_2^T \dots \mU_L^T \mH \mU_L \dots \mU_2 \mU_1,$$
where the $\mH$ and $\mU_1, \dots, \mU_L$ matrices conform to the following constraints: 
(i) Each $\mU_\ell$ is an orthogonal matrix that is a $k$-point rotation for some small $k$, 
meaning that it only rotates $k$ coordinates at a time; 
(ii) There is a nested sequence of sets $\sS_L \subseteq \cdots \subseteq \sS_1 \subseteq \sS_0 = [n]$ 
such that the coordinates rotated by $\mU_\ell$ are a subset of $\sS_\ell$; and 
(iii) $\mH$ is an $\sS_L$-core-diagonal matrix meaning that is diagonal with a an additional small 
$\sS_L\times \sS_L$ dimensional ``core''. 
Finding the best MMF factorization to a symmetric matrix $\mA$ involves solving
\begin{equation}
\min_{\substack{\sS_L \subseteq \cdots \subseteq \sS_1 \subseteq \sS_0 = [n] \\ \mH \in \sH^{\sS_L}_n; \mU_1, \dots, \mU_L \in \sO}} || \mA - \mU_1^T \dots \mU_L^T \mH \mU_L \dots \mU_1 ||.
\label{eq:mmf-opt}
\end{equation}
Assuming that we measure error in the Frobenius norm, (\ref{eq:mmf-opt}) is equivalent to
\begin{equation}
\min_{\substack{\sS_L \subseteq \cdots \subseteq \sS_1 \subseteq \sS_0 = [n] \\ \mU_1, \dots, \mU_L \in \sO}} || \mU_L \dots \mU_1 \mA \mU_1^T \dots \mU_L^T ||^2_{\text{resi}},
\label{eq:mmf-resi}
\end{equation}
where $||\cdot||_{\text{resi}}^2$ is the squared residual norm 
$||\mH||_{\text{resi}}^2 = \sum_{i \neq j; (i, j) \not\in \sS_L \times \sS_L} |\mH_{i, j}|^2$. 
There are two fundamental difficulties in MMF optimization: 
finding the optimal nested sequence of $\sS_\ell$ is a combinatorially hard 
(e.g., there are ${d_\ell \choose k}$ ways to choose $k$ indices out of $\sS_\ell$); 
and the solution for $\mU_\ell$ must satisfy the orthogonality constraint such that 
$\mU_\ell^T \mU_\ell = \mI$. 
The existing literature on solving this optimization problem 
\citep{DBLP:conf/icml/KondorTG14} \citep{pmlr-v51-teneva16} \citep{8099564} 
\citep{NIPS2017_850af92f} has various heuristic elements and has a number of limitations: 
\vspace{-5pt}
\begin{itemize}
\setlength{\itemsep}{0pt}
\item There is no guarantee that the greedy heuristics (e.g., clustering) used in selecting 
$k$ rows/columns $\sI_\ell = \{i_1, .., i_k\} \subset \sS_\ell$ for each rotation  
return a globally optimal factorization.
\item Instead of direct optimization for each rotation 
$\mU_\ell \triangleq \mI_{n - k} \oplus_{\sI_\ell} \mO_\ell$ where $\mO_\ell \in \sS\sO(k)$ 
globally and simultaneously with the objective (\ref{eq:mmf-opt}), 
Jacobi MMFs (see Proposition 2 of \citep{DBLP:conf/icml/KondorTG14}) apply the greedy strategy of optimizing 
them locally and sequentially. 
Again, this does not necessarily lead to a \textit{globally} optimal combination of rotations. 
\item 
Most MMF algorithms are limited to the simplest case of $k = 2$ where 
$\mU_\ell$ is just a Given rotation, which 
can be parameterized by a single variable, the rotation angle $\theta_\ell$. 
This makes it possible to optimize the greed objective by simple gradient descent, but 
larger rotations would yield more expressive factorizations and better approximations.
\end{itemize}
\vspace{-5pt}
In contrast, we propose an iterative algorithm to directly optimize the global MMF objective 
(\ref{eq:mmf-opt}): 
\begin{itemize}
\item We use gradient descent algorithm on the Stiefel manifold to optimize all rotations 
$\{\mU_\ell\}_{\ell = 1}^L$ \textit{simultaneously}, whilst satisfying the orthogonality constraints. 
Importantly, the Stiefel manifold optimization is not limited to $k = 2$ case 
(Section ~\ref{sec:stiefel}).
\item We formulate the problem of finding the optimal nested sequence $\sS_L \subseteq \cdots \subseteq \sS_1 \subseteq \sS_0 = [n]$ as learning a Markov Decision Process (MDP) that can be subsequently solved by the gradient policy method of Reinforcement Learning (RL), in which the RL agent (or stochastic policy) is modeled by graph neural networks (GNN) (Section ~\ref{sec:RL}).
\end{itemize}
We show that the resulting learning-based MMF algorithm outperforms existing greedy MMFs and other 
traditional baselines for matrix approximation in various scenarios (see Section \ref{sec:Experiments}).


%% file: Stiefel_Manifold_Optimization.tex
\section{Stiefel Manifold Optimization} \label{sec:stiefel}

The MMF optimization problem in (\ref{eq:mmf-opt}) and (\ref{eq:mmf-resi}) is equivalent to
\begin{equation}
\min_{\sS_L \subseteq \cdots \subseteq \sS_1 \subseteq \sS_0 = [n]} \min_{\mU_1, \dots, \mU_L \in \sO} || \mU_L \dots \mU_1 \mA \mU_1^T \dots \mU_L^T ||^2_{\text{resi}},
\label{eq:mmf-two-phases}
\end{equation}
In order to solve the inner optimization problem of (\ref{eq:mmf-two-phases}), 
we consider the following generic optimization with orthogonality constraints:
\begin{equation}
\min_{\mX \in \R^{n \times p}} \mathcal{F}(\mX), \ \ \text{s.t.} \ \ \mX^T \mX = \mI_p,
\label{eq:opt-prob}
\end{equation}
where $\mI_p$ is the identity matrix and $\mathcal{F}(\mX): \R^{n \times p} \rightarrow \R$ 
is a differentiable function. 
The feasible set $\mathcal{V}_p(\R^n) = \{\mX \in \R^{n \times p}: \mX^T \mX = \mI_p\}$ is referred to as the 
Stiefel manifold of $p$ orthonormal vectors in $\R^{n}$ that has dimension equal to $np - \frac{1}{2}p(p + 1)$. 
We will view $\mathcal{V}_p(\R^n)$ as an embedded submanifold of $\R^{n \times p}$. \\ \\
When there is more than one orthogonal constraint, (\ref{eq:opt-prob}) is written as
\begin{equation}
\min_{\mX_1 \in \mathcal{V}_{p_1}(\R^{n_1}), \dots, \mX_q \in \mathcal{V}_{p_q}(\R^{n_q})} \mathcal{F}(\mX_1, \dots, \mX_q)
\label{eq:opt-prob-extended}
\end{equation}
where there are $q$ variables with corresponding $q$ orthogonal constraints. 
For example, in the MMF optimization problem (\ref{eq:mmf-opt}), suppose we are already given $\sS_L \subseteq \cdots \subseteq \sS_1 \subseteq \sS_0 = [n]$ meaning that the indices of active rows/columns at each resolution were already determined, for simplicity. In this case, we have $q = L$ number of variables such that each variable $\mX_\ell = \mO_\ell \in \R^{k \times k}$, where $\mU_\ell = \mI_{n - k} \oplus_{\sI_\ell} \mO_\ell \in \R^{n \times n}$ in which $\sI_\ell$ is a subset of $k$ indices from $\sS_\ell$, must satisfy the orthogonality constraint. The corresponding objective function is 
\begin{equation}
\mathcal{F}(\mO_1, \dots, \mO_L) = || \mU_L \dots \mU_1 \mA \mU_1^T \dots \mU_L^T ||^2_{\text{resi}}.
\label{eq:mmf-core}
\end{equation}
Details about Stiefel manifold optimization is included in the Appendix.

%% file: Reinforcement_Learning.tex
\section{Reinforcement Learning} \label{sec:RL}

\subsection{Problem formulation} \label{sec:RL-problem}

We formulate the problem of finding the optimal nested sequence of sets 
$\sS_L \subseteq \cdots \subseteq \sS_1 \subseteq \sS_0 = [n]$ as learning an RL agent in order to solve 
the MMF optimization in (\ref{eq:mmf-opt}). 
There are two fundamental parts to index selection for each resolution level $\ell \in \{1, .., L\}$:
\vspace{-5pt}
\begin{itemize}
\setlength{\itemsep}{0pt}
\item Select $k$ indices $\sI_\ell = \{i_1, .., i_k\} \subset \sS_{\ell - 1}$ to construct the corresponding rotation matrix $\mU_\ell$ (see Section ~\ref{sec:stiefel}).
\item Select the set of indices $\sT_\ell \subset \sS_{\ell - 1}$ of rows/columns that are to be wavelets at 
this level, and then be eliminated by setting $\sS_\ell = \sS_{\ell - 1} \setminus \sT_\ell$. 
To reduce the computational cost, we assume that each resolution level has only one row/column to be selected as the wavelet (e.g., a single wavelet) such that $|\sT_\ell| = 1$. That means the cardinality of $\sS_\ell$ reduces by $1$ after each level, $d_\ell = n - \ell$, and size of the core block of $\mH$ is $(n - L) \times (n - L)$ that corresponds to exactly $n - L$ active rows/columns at the end.
\end{itemize}

\subsection{Markov Decision Process} \label{sec:MDP}

A key task for building our model is to specify our index selection procedure. We design an iterative index selection process and formulate it as a general decision process $M = (S, A, P, R, \gamma)$ as follows. \\ \\
$S$ is the set of states (or state space) that consists of all possible intermediate and final states in which each state $s \in S$ is a tuple of $(\overline{\mA}, \sS, \ell)$ where $\ell$ indicates the resolution level, $\sS$ indicates the set of active row/column indices, and $\overline{\mA} = \mA_{\sS, \sS}$ indicates the sub-matrix of $\mA$ with indices of rows and columns are from $\sS$ (e.g., $|\sS| = d_\ell$, $\overline{\mA} \in \R^{d_\ell \times d_\ell}$). We start at state $s_0 = (\mA, [n], 0)$ where $\mA$ is the input matrix that MMF tries to factorize (e.g., no rows/columns removed yet) and $[n]$ indicates all rows/columns are still active. The set of terminal (final) states $S^* \subset S$ includes every state $s^*$ that has $\ell = L$. \\ \\
$A$ is the set of actions that describe the modification made to current state at each time step. An action $a \in A$ validly applied to a non-terminal state $s = (\mA_{\sS, \sS}, \sS, \ell)$ is a tuple $(\sI, \sT)$ where $\sI = \{i_1, .., i_k\} \subset \sS$ is the set of $k$ indices corresponding to the rotation matrix $\mU_{\ell + 1}$, and $\sT \subset \sS$ is the set of wavelet indices to spit out at this level. This action transforms the state into the next one $s' = (\mA_{\sS', \sS'}, \sS', \ell + 1)$ where $\sS' = \sS \setminus \sT$ meaning the set of active indices gets shrinked further. The action is called invalid for the current state if and only if $\sI \not \subset \sS$ or $\sT \not \subset \sS$. \\ \\
$P$ is the transition dynamics that specifies the possible outcomes of carrying out an action at time $t$, $p(s_{\ell + 1}|s_\ell, .., s_0, a_\ell)$, as a conditional probability on the sequence of previous states $(s_0, .., s_\ell)$ and the action $a_\ell$ applied to state $s_\ell$. Basically, the RL environment carries out actions that obey the given action rules. Invalid actions proposed by the policy network are rejected and the state remains unchanged. The state transition distribution is constructed as
$$p(s_{\ell + 1}|s_\ell, .., s_0) = \sum_{a_\ell} p(a_\ell|s_\ell, .., s_0) p(s_{\ell + 1}|s_\ell, .., s_0, a_\ell),$$
where $p(a_\ell|s_\ell, .., s_0)$ is represented as a parameterized policy network $\pi_\theta$ with learnable parameters $\theta$. Markov Decision Process (MDP) requires the state transition dynamics to satisfy the Markov property: $p(s_{\ell + 1}|s_\ell, .., s_0) = p(s_{\ell + 1}|s_\ell)$. Under this property, the policy network $\pi_\theta(a_\ell|s_\ell)$ only needs the intermediate state $s_\ell$ to derive an action. The whole trajectory is always started by the same $s_0$ and finished by a terminal state after exactly $L$ transitions as depicted as follows:
$$s_0 \xrightarrow{~~a \sim \pi_\theta(\cdot|s_0)~~} s_1 \xrightarrow{~~a \sim \pi_\theta(\cdot|s_1)~~} s_2 \dots s_{L - 1} \xrightarrow{~~a \sim \pi_\theta(\cdot|s_{L - 1})~~} s_L \in S^*.$$
Series of actions recorded along the trajectory allows us to easily construct the nested sequence $\sS_L \subseteq \cdots \subseteq \sS_1 \subseteq \sS_0 = [n]$. \\ \\
$R(s^*)$ is the reward function that specifies the reward after reaching a terminal state $s^*$. The reward function is defined as negative of the MMF reconstruction loss such that
\begin{equation}
R(s^*) = - || \mA - \mU_1^T \dots \mU_L^T \mH \mU_L \dots \mU_1 ||_F.
\label{eq:final-reward}
\end{equation}
We want to maximize this final reward that is equivalent to minimize error of MMF in Frobenius norm (as in problem (\ref{eq:mmf-opt})). Evaluation of the reward requires the Stiefel manifold optimization (see Section ~\ref{sec:stiefel}) for rotations $\{\mU_\ell\}_{\ell = 1}^L$. Obviously, the final reward in Eq.~(\ref{eq:final-reward}) is the most important. However, to improve the training quality of the policy, we can define the intermediate reward $R(s)$ for non-terminal states $s = (\mA_{\sS, \sS}, \sS, \ell) \not \in S^*$ as the immediate improvement of the $\ell$-th resolution ($L > \ell > 0$):
\begin{equation}
R(s) = - ||[\mU_\ell \mA_{\ell - 1} \mU_\ell^T]_{\sS, \sS}||_{\text{resi}}^2.
\label{eq:immediate-reward}
\end{equation}
Along the trajectory $(s_0, s_1, .., s_L)$, we generate the corresponding sequence of rewards 
$(r_1, r_2, .., r_L)$ based on (\ref{eq:final-reward}, \ref{eq:immediate-reward}). \\ \\
$\gamma$ is the discount factor, a penalty to uncertainty of future rewards, $0 < \gamma \leq 1$. We define the return or discounted future reward $g_\ell$ for $\ell = 0, .., L - 1$ as
\begin{equation}
g_\ell = \sum_{k = 0}^{L - \ell - 1} \gamma^k r_{\ell + k + 1},
\label{eq:return}
\end{equation}
which in the case of $\gamma = 1$ indicates simply accumulating all the immediate rewards and the final reward along the trajectory. 

\subsection{Graph convolutional policy network} \label{sec:gnn}

In this section, we design our policy network $\pi_\theta$ as a graph neural network (GNN) 
with the message passing scheme. 
We consider the symmetric matrix $\mA$ being represented by a weighted undirected graph $\mathcal{G} = (V, E)$ in which $V$ is the set of nodes such that each node corresponds to a row/column of $\mA$, and $E$ is the set of edges such that the edge $(i, j) \in E$ has the weight $\mA_{i, j}$. As defined in Section ~\ref{sec:MDP}, a state $s \in S$ is a tuple $(\mA_{\sS, \sS}, \sS, \ell)$ in which $\mA_{\sS, \sS}$ is the sub-matrix restricted to the active rows/columns $\sS$, and an action $a \in A$ is a tuple $(\sI, \sT)$ in which $\sI$ is the set of $k$ indices corresponding to the $(\ell + 1)$-th rotation and $\sT$ is the set of indices to spit out as wavelets. Practically, a state can be simply represented by a single binary vector such that if a bit is $1$ then the corresponding index is active, without the need to explicitly storage matrix $\mA_{\sS, \sS}$ that can be efficiently constructed from $\mA$ by any numerical toolkit. Our GNN policy network $\pi_\theta(a|s)$ learns to encode the underlying graph represented by $\mA_{\sS, \sS}$ and returns a sample of valid action such that $\sT \subset \sI \subset \sS$. In Section ~\ref{sec:RL-problem}, we assume that $\sT$ contains only a single index that we will call as the \textit{pivot} $i^*$ (e.g., $\sT = \{i^*\}$). Thus, the task of our GNN model is to learn to select the pivot $i^*$ first, and then select the rest $k - 1$ indices of $\sI$ that are highly correlated to the pivot. \\ \\
The simplest implementation of GNNs is Message Passing Neural Networks (MPNNs) \citep{DBLP:journals/corr/GilmerSRVD17}. Suppose that the node embeddings (messages) $\mM_0 \in \R^{N \times F}$ are initialized by the input node features where $N$ is the number of nodes and $F$ is the number of features for each node. Iteratively, the messages are propagated from each node to its neighborhood, and then transformed by a combination of linear transformations and non-linearities, e.g.,
\begin{equation}
\hat{\mM}_t = \mA\mM_{t - 1}, \ \ \ \ \mM_t = \sigma(\hat{\mM}_t \mW_{t - 1}),
\label{eq:mpnn}
\end{equation}
where $\mA$ is the adjacency matrix; $\hat{\mM}_t$ and $\mM_t \in \R^{N \times D}$ are the 
aggregated messages (by summing over each node's neighborhood) and the output messages at the $t$'th iteration, respectively; $\sigma$ is a element-wise non-linearity function (e.g., sigmoid, ReLU, etc.); and $\mW$s are learnable weight matrices such that $\mW_0 \in \R^{F \times D}$ and $\mW_t \in \R^{D \times D}$ for $t > 0$. Basically, the set of learnable parameters $\theta$ of our policy network $\pi$ includes all $\mW$s. In some cases, a graph Laplacian $\mL$ is used instead of the adjacency matrix $\mA$ in model (\ref{eq:mpnn}), for example, graph Laplacian $\mL = \mD^{-1}\mA$ or its symmetric normalized version $\tilde{\mL} = \mI - \mD^{-1/2} \mA \mD^{-1/2}$. One way to incorporate the set of active rows/columns/nodes $\sS$ into our GNN model is by initializing the input node feature with a binary label ($f = 1$) such that a node $v$ has label $1$ if the $v$-th row/column is still active, otherwise $0$. For a more efficient implementation, we can execute the message passing in the block $\mA_{\sS, \sS}$ only. Supposing that the message passing scheme is executed for $T$ iterations, we concatenate messages from every iteration together into the final embedding:
\begin{equation}
\mM = \bigoplus_{t = 1}^T \mM_t \in \R^{N \times DT}.
\label{eq:final-message}
\end{equation}
Model (\ref{eq:mpnn}) produces the embedding for each node that allows us to define a sampling procedure to select the pivot $i^*$. Given the final embedding $\mM$ from Eq.~(\ref{eq:final-message}), we define the probability $\mP_i$ that node $i \in \sS$ is being selected as the pivot as:
$$P_i = \frac{\exp(\hat{P}_i)}{\sum_{j \in \sS} \exp(\hat{P}_j)}, \ \ \ \ \text{where} \ \ \ \ \hat{P}_i = \sum_f \mM_{i, f}.$$
In order to make the sampling procedure differentiable for backpropagation, we apply the Gumbel-max trick \citep{Gumbel1954} \citep{NIPS2014_309fee4e} \citep{jang2017categorical} that provides a simple and efficient way to draw sample $i^*$ as follows:
$$i^* = \text{one-hot}\big(\argmax_{i \in \sS} \big[ G_i + \log P_i \big]\big),$$
where $G_i$ are i.i.d samples drawn from $\text{Gumbel}(0, 1)$. Technically, the sample is represented by a one-hot vector such that the $i^*$-th element is $1$. Similarly, $\sT$, $\sI$ and $\sS$ are represented by vectors in $\{0, 1\}^N$ in which a $1$-element indicates the existence of the corresponding index in the set. Furthermore, the set union and minus operations (e.g., $\sS \setminus \sT$) can be easily done by vector addition and subtraction, respectively. \\ \\
Given the pivot $i^*$, we compute the similarity score between $i^*$ and other nodes $i \in \sS$ as $C_i = \langle \mM_{i^*, :}, \mM_{i, :} \rangle$. Finally, we sample $k - 1$ nodes with the highest similarity scores to $i$ sequentially (one-by-one) without replacement by the Gumbel-max trick, that completes our sampling procedure for action $a = (\sI, \sT)$. \\ \\
The REINFORCE \citep{SCC.Williams1988} \citep{10.1007/BF00992696} \citep{NIPS1999_464d828b} update rule for policy parameters is 
\begin{equation}
\theta \leftarrow \theta + \eta \gamma^\ell g_\ell \nabla_\theta \log \pi_\theta(a_\ell|s_\ell) \ \ \ \ \text{for} \ \ \ell = 0, .., L - 1; \label{eq:REINFORCE-rule}
\end{equation}
where $\eta$ is the learning rate, that is used in training our policy network $\pi_\theta$ in Algorithm ~\ref{alg:final} (see section \textit{Policy gradient methods} in the Appendix).

\subsection{The learning algorithm} \label{sec:learing-algorithm}

Putting everything together, our MMF learning algorithm is sketched in Algorithm \ref{alg:final}. Iteratively: (1) we sample a trajectory by running the policy network that indicates the indices for rotation and wavelet for each resolution, (2) we apply the Stiefel manifold optimization to find the rotations, and (3) we compute the future rewards and update the parameters of the policy network by REINFORCE accordingly. 
The learning terminates when the average error over a window of size $\omega$ iterations increases.
\begin{algorithm*}
\caption{MMF learning algorithm optimizing problem (\ref{eq:mmf-opt})} \label{alg:final}
\begin{algorithmic}[1]
\State \textbf{Input:} Matrix $\mA$, number of resolutions $L$, and constants $k$, $\gamma$, $\eta$, and $\omega$.
\State Initialize the policy parameter $\theta$ at random.
\While{true}
\State Start from state $s_0$ \Comment{$s_0 \triangleq (\mA, [n], 0)$}
\State Initialize $\sS_0 \leftarrow [n]$ \Comment{All rows/columns are active at the beginning.}
\For{$\ell = 0, .., L - 1$}
	\State Sample action $a_\ell = (\sI_{\ell + 1}, \sT_{\ell + 1})$ from $\pi_\theta(a_\ell|s_\ell)$. \Comment{See Section \ref{sec:gnn}.}
	\State $\sS_{\ell + 1} \leftarrow \sS_\ell \setminus \sT_{\ell + 1}$ \Comment{Eliminate the wavelet index (indices).}
	\State $s_{\ell + 1} \leftarrow (\mA_{\sS_{\ell + 1}, \sS_{\ell + 1}}, \sS_{\ell + 1}, \ell + 1)$ \Comment{New state with a smaller active set.}
\EndFor
\State Given $\{\sI_\ell\}_{\ell = 1}^L$, minimize objective (\ref{eq:mmf-core}) by Stiefel manifold optimization to find $\{\mO_\ell\}_{\ell = 1}^L$. \Comment{$\mU_\ell = \mI_{n - k} \oplus_{\sI_\ell} \mO_\ell$}
\For{$\ell = 0, .., L - 1$}
	\State Estimate the return $g_\ell$ based on Eq.~(\ref{eq:final-reward}), Eq.~(\ref{eq:immediate-reward}), and Eq.~(\ref{eq:return}).
	\State $\theta \leftarrow \theta + \eta \gamma^\ell g_\ell \nabla_\theta \log \pi_\theta(a_\ell|s_\ell)$ \Comment{REINFORCE policy update in Eq.~(\ref{eq:REINFORCE-rule})}
\EndFor
\State Terminate if the average error of the last $\omega$ iterations increases. \Comment{Early stopping.}
\EndWhile
\end{algorithmic}
\end{algorithm*}

%% file: Wavelet_Networks.tex
\section{Wavelet Networks on Graphs} \label{sec:networks}

\subsection{Motivation}

The eigendecomposition of the normalized graph Laplacian operator $\tilde{\mL} = \mU^T \mH \mU$ 
can be used as the basis of a graph Fourier transform. \citep{6494675} defines graph Fourier transform (GFT) on a graph $\mathcal{G} = (V, E)$ of a graph signal $\vf \in \R^n$ (that is understood as a function $f: V \rightarrow \R$ defined on the vertices of the graph) as $\hat{\vf} = \mU^T \vf$, and the inverse graph Fourier transform as $\vf = \mU \hat{\vf}$. Analogously to the classical Fourier transform, GFT provides a way to represent a graph signal in two domains: the vertex domain and the graph spectral domain; to filter graph signal according to smoothness; and to define the graph convolution operator, denoted as $*_{\mathcal{G}}$:
\begin{equation}
\vf *_{\mathcal{G}} \vg = \mU \big( (\mU^T \vg) \odot (\mU^T \vf) \big),
\label{eq:gft-conv}
\end{equation}
where $\vg$ denotes the convolution kernel, and $\odot$ is the element-wise Hadamard product. 
If we replace the vector $\mU^T \vg$ by a diagonal matrix $\tilde{\vg}$, then we can rewrite the Hadamard 
product in Eq.~(\ref{eq:gft-conv}) to matrix multiplication as $\mU \tilde{\vg} \mU^T \vf$ (that is understood as filtering the signal $\vf$ by the filter $\tilde{\vg}$). Based on GFT, \citep{ae482107de73461787258f805cf8f4ed} and \citep{10.5555/3157382.3157527} construct convolutional neural networks (CNNs) learning on spectral domain for discrete structures such as graphs. However, there are two fundamental limitations of GFT:
\begin{itemize}
\item High computational cost: eigendecomposition of the graph Laplacian has complexity $O(n^3)$, 
and ``Fourier transform'' itself involves multiplying the signal with a dense matrix of eigenvectors. 
\item The graph convolution is not localized in the vertex domain, even 
if the graph itself has well defined local communities.
\end{itemize}
To address these limitations, we propose a modified spectral graph network based on the MMF wavelet 
basis rather than the eigenbasis of the Laplacian. 
This has the following advantages: 
(i) the wavelets are generally localized in both vertex domain and frequency, 
(ii) the individual basis transforms are sparse, and 
(iii) MMF provides a computationally efficient way of decomposing graph signals into components at 
different granularity levels and an excellent basis for sparse approximations.

\subsection{Network construction}


In the case $\mA$ is the normalized graph Laplacian of a graph $\mathcal{G} = (V, E)$, the wavelet transform (up to level $L$) expresses a graph signal (function over the vertex domain) $f: V \rightarrow \R$, without loss of generality $f \in \sV_0$, as:
$$f(v) = \sum_{\ell = 1}^L \sum_m \alpha_m^\ell \psi_m^\ell(v) + \sum_m \beta_m \phi_m^L(v), \ \ \ \ \text{for each} \ \ v \in V,$$ 
where $\alpha_m^\ell = \langle f, \psi_m^\ell \rangle$ and $\beta_m = \langle f, \phi_m^L \rangle$ are the wavelet coefficients. At each level, a set of coordinates $\sT_\ell \subset \sS_{\ell -  1}$ are selected to be the wavelet indices, and then to be eliminated from the active set by setting $\sS_\ell = \sS_{\ell - 1} \setminus \sT_\ell$ (see Section \ref{sec:RL-problem}). Practically, we make the assumption that we only select $1$ wavelet index for each level (see Section \ref{sec:RL-problem}) 
that results in a single mother wavelet $\psi^\ell = [\mA_\ell]_{i^*, :}$ where $i^*$ is the selected index (see Section \ref{sec:gnn}). We get exactly $L$ mother wavelets $\overline{\psi} = \{\psi^1, \psi^2, \dots, \psi^L\}$. On the another hand, the active rows of $\mH = \mA_L$ make exactly $N - L$ father wavelets $\overline{\phi} = \{\phi^L_m = \mH_{m, :}\}_{m \in \sS_L}$. In total, a graph of $N$ vertices has exactly $N$ wavelets (both mothers and fathers). Analogous to the convolution based on GFT \citep{ae482107de73461787258f805cf8f4ed}, each convolution layer $k = 1, .., K$ of our wavelet network transforms an input vector $\vf^{(k - 1)}$ of size $|V| \times F_{k - 1}$ into an output $\vf^{(k)}$ of size $|V| \times F_k$ as
\begin{equation}
\vf^{(k)}_{:, j} = \sigma \bigg( \mW \sum_{i = 1}^{F_{k - 1}} \vg^{(k)}_{i, j} \mW^T \vf^{(k - 1)}_{:, i} \bigg) \ \ \ \ \text{for} \ \ j = 1, \dots, F_k,
\label{eq:wavevlet-conv}
\end{equation}
where $\mW$ is our wavelet basis matrix as we concatenate $\overline{\phi}$ and $\overline{\psi}$ column-by-column, $\vg^{(k)}_{i, j}$ is a parameter/filter in the form of a diagonal matrix learned in spectral domain, and $\sigma$ is an element-wise linearity (e.g., ReLU, sigmoid, etc.). \\ \\
For example, in node classification tasks, assume the number of classes is $C$, 
the set of labeled nodes is $V_{\text{label}}$, 
and we are given a normalized graph Laplacian $\tilde{L}$ and an input node feature matrix $\vf^{(0)}$. 
First of all, we apply our MMF learning algorithm \ref{alg:final} to factorize $\tilde{L}$ and produce our wavelet basis matrix $\mW$. Then, we construct our wavelet network as a multi-layer CNNs with each convolution is defined as in Eq.~(\ref{eq:wavevlet-conv}) that transforms $\vf^{(0)}$ into $\vf^{(K)}$ after $K$ layers. The top convolution layer $K$-th returns exactly $F_K = C$ features and uses softmax instead of the nonlinearity $\sigma$ for each node. The loss is the cross-entropy error over all labeled nodes as:
\begin{equation}
\mathcal{L} = - \sum_{v \in V_{\text{label}}} \sum_{c = 1}^C \vy_{v, c} \ln \vf^{(K)}_{v, c},
\label{eq:entropy-loss}
\end{equation}
where $\vy_{v, c}$ is a binary indicator that is equal to $1$ if node $v$ is labeled with class $c$, and $0$ otherwise. The set of weights $\{\vg^{(k)}\}_{k = 1}^K$ are trained using gradient descent optimizing the loss in Eq.~(\ref{eq:entropy-loss}).

%% file: Experiments.tex
\input{figure_matrix}
\section{Experiments} \label{sec:Experiments}

\subsection{Matrix factorization} \label{sec:approx-exp}

We evaluate the performance of our MMF learning algorithm in comparison with the original greedy algorithm \citep{DBLP:conf/icml/KondorTG14} and the Nystr\"{o}m method \citep{pmlr-v28-gittens13} in the task of matrix factorization on 3 datasets: (i) normalized graph Laplacian of the Karate club network ($N = 34$, $E = 78$) \citep{Karate}; (ii) a Kronecker product matrix ($N = 512$), $\mathcal{K}_1^n$, of order $n = 9$, where $\mathcal{K}_1 = ((0, 1), (1, 1))$ is a $2 \times 2$ seed matrix \citep{JMLR:v11:leskovec10a}; and (iii) normalized graph Laplacian of a Cayley tree or Bethe lattice with coordination number $z = 4$ and $4$ levels of depth ($N = 161$). The rotation matrix size $K$ are $8$, $16$ and $8$ for Karate, Kronecker and Cayley, 
respectively. Meanwhile, the original greedy MMF is limited to $K = 2$ and implements an exhaustive search to find an optimal pair of indices for each rotation. For both versions of MMF, we drop $c = 1$ columns after each rotation, 
which results in a final core size of $d_L = N - c \times L$. 
The exception is for the Kronecker matrix ($N = 512$), our learning algorithm drops up to $8$ columns (for example, $L = 62$ and $c = 8$ results into $d_L = 16$) to make sure that the number of learnable parameters $L \times K^2$ is much smaller the matrix size $N^2$. 
Our learning algorithm compresses the Kronecker matrix down to $6-7\%$ of its original size. The details of efficient training reinforcement learning with the policy networks implemented by GNNs are included in the Appendix. \\ \\
For the baseline of Nystr\"{o}m method, we randomly select, by uniform sampling without replacement, the same number $d_L$ columns $\mC$ from $\mA$ and take out $\mW$ as the corresponding $d_L \times d_L$ submatrix of $\mA$. The Nystr\"{o}m method approximates $\mA \approx \mC\mW^{\dagger}\mC^T$. We measure the approximation error in Frobenius norm. Figure \ref{fig:matrix} shows our MMF learning algorithm consistently outperforms the original greedy algorithm and the Nystr\"{o}m baseline given the same number of active columns, $d_L$. Figure \ref{fig:wavelets-visual} depicts the wavelet bases at different levels of resolution.

\input{table_node}
\input{table_graph}
\input{figure_wavelets}
\subsection{Node classification on citation graphs}

To evaluate the wavelet bases returned by our learnable MMF algorithm, we construct our wavelet networks 
(WNNs) as in Sec.~\ref{sec:networks} and apply it to the task of node classification on two citation graphs, 
Cora $(N = 2,708)$ and Citeseer $(N = 3,312)$ \citep{PSen2008} 
in which nodes and edges represent documents and citation links. 
Each document in Cora and Citeseer has an associated feature vector (of length  $1,433$ resp.~$3,703$)   
computed from word frequencies, and is classified into one of $7$ and $6$ classes, respectively. 
We factorize the normalized graph Laplacian by learnable MMF with $K = 16$ to obtain the wavelet bases. 
The resulting MMF wavelets are sparse, which makes it possible to run a fast transform 
on the node features by sparse matrix multiplication: only $4.69\%$ and $15.25\%$ of elements are non-zero in Citeseer and Cora, 
respectively. 
In constrast, Fourier bases given by eigendecomposition of the graph Laplacian are completely dense 
($100\%$ of elements are non-zero). 
We evaluate our WNNs with 3 different random splits of train/validation/test: (1) $20\%$/$20\%$/$60\%$ 
denoted as MMF$_1$, (2) $40\%$/$20\%$/$40\%$ denoted as MMF$_2$, and (3) $60\%$/$20\%$/$20\%$ denoted as MMF$_3$. 
The WNN learns to encode the whole graph with $6$ layers of spectral convolution and $100$ hidden dimensions 
for each node. During training, the network is only trained to predict the node labels in the training set. 
Hyperparameter searching is done on the validation set. 
The number of epochs is $256$ and we use the Adam optimization method \citep{Adam} 
with learning rate $\eta = 10^{-3}$. 
We report the final test accuracy for each split in Table \ref{tbl:node-classification}. \\ \\
We compare with several traditional methods and deep learning methods including other spectral graph convolution networks such as Spectral CNN, and graph wavelet neural networks (GWNN). Baseline results are taken from \citep{xu2018graph}. Our wavelet networks perform competitively against state-of-the-art methods in the field.

\subsection{Graph classification}

We also tested our WNNs on standard graph classification benchmarks including four bioinformatics datasets: 
(1) MUTAG, which is a dataset of 188 mutagenic aromatic and heteroaromatic nitro compounds with 7 discrete 
labels \citep{doi:10.1021/jm00106a046}; (2) PTC, which consists of 344 chemical compounds with 19 discrete 
labels that have been tested for positive or negative toxicity in lab rats \citep{10.1093/bioinformatics/btg130}; 
(3) PROTEINS, which contains 1,113 molecular graphs with binary labels, 
where nodes are secondary structure elements (SSEs) and there is an edge between two nodes if they are 
neighbors in the amino-acid sequence or in 3D space \citep{ProteinKernel}; 
(4) NCI1, which has 4,110 compounds with binary labels, each screened for activity against small cell 
lung cancer and ovarian cancer lines \citep{NCIDataset}. 
Each molecule is represented by an adjacency matrix, and we represent each atomic type as a one-hot vector 
and use them as the node features. \\ \\
We factorize all normalized graph Laplacian matrices in these datasets by MMF with $K = 2$ to obtain the wavelet bases. 
Again, MMF wavelets are \textbf{sparse} and suitable for fast transform via sparse matrix multiplication, 
with the following average percentages of non-zero elements for each dataset: 
$19.23\%$ (MUTAG), $18.18\%$ (PTC), $2.26\%$ (PROTEINS) and $11.43\%$ (NCI1). \\ \\
Our WNNs contain 6 layers of spectral convolution, 32 hidden units for each node, 
and are trained with 256 epochs by Adam optimization with an initial learning rate of $10^{-3}$. 
We follow the evaluation protocol of 10-fold cross-validation from \citep{Zhang2018AnED}. 
We compare our results to several deep learning methods and popular graph kernel methods. Baseline results are taken from \citep{maron2018invariant}.
Our WNNs outperform 7/8, 7/8, 8/8, and 2/8 baseline methods on MUTAG, PTC, PROTEINS, and NCI1, 
respectively (see Table \ref{tbl:graph-classification}).

%% file: figure_matrix.tex
\begin{figure*}
\begin{center}
\includegraphics[width=0.30\textwidth]{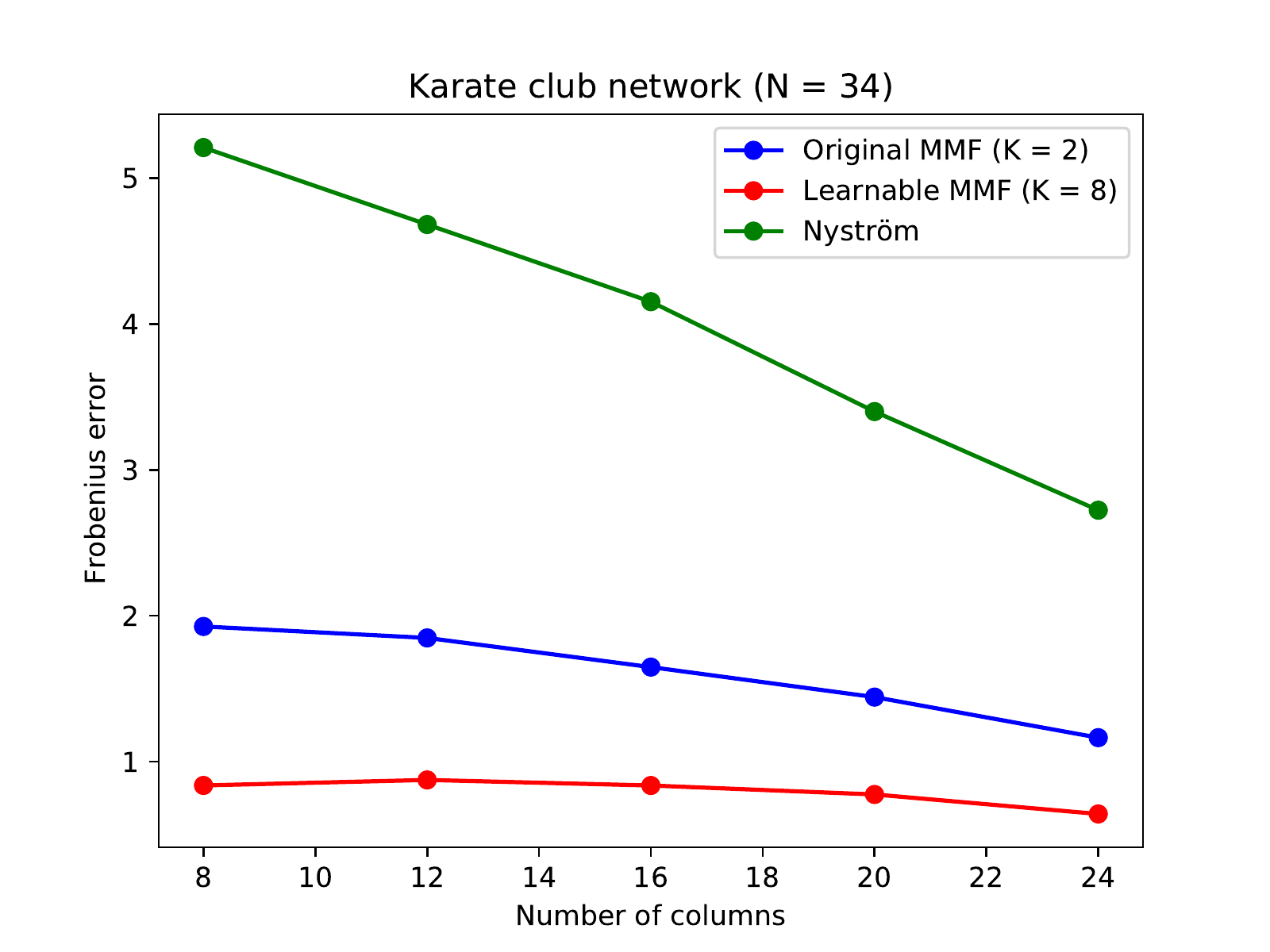}
\includegraphics[width=0.30\textwidth]{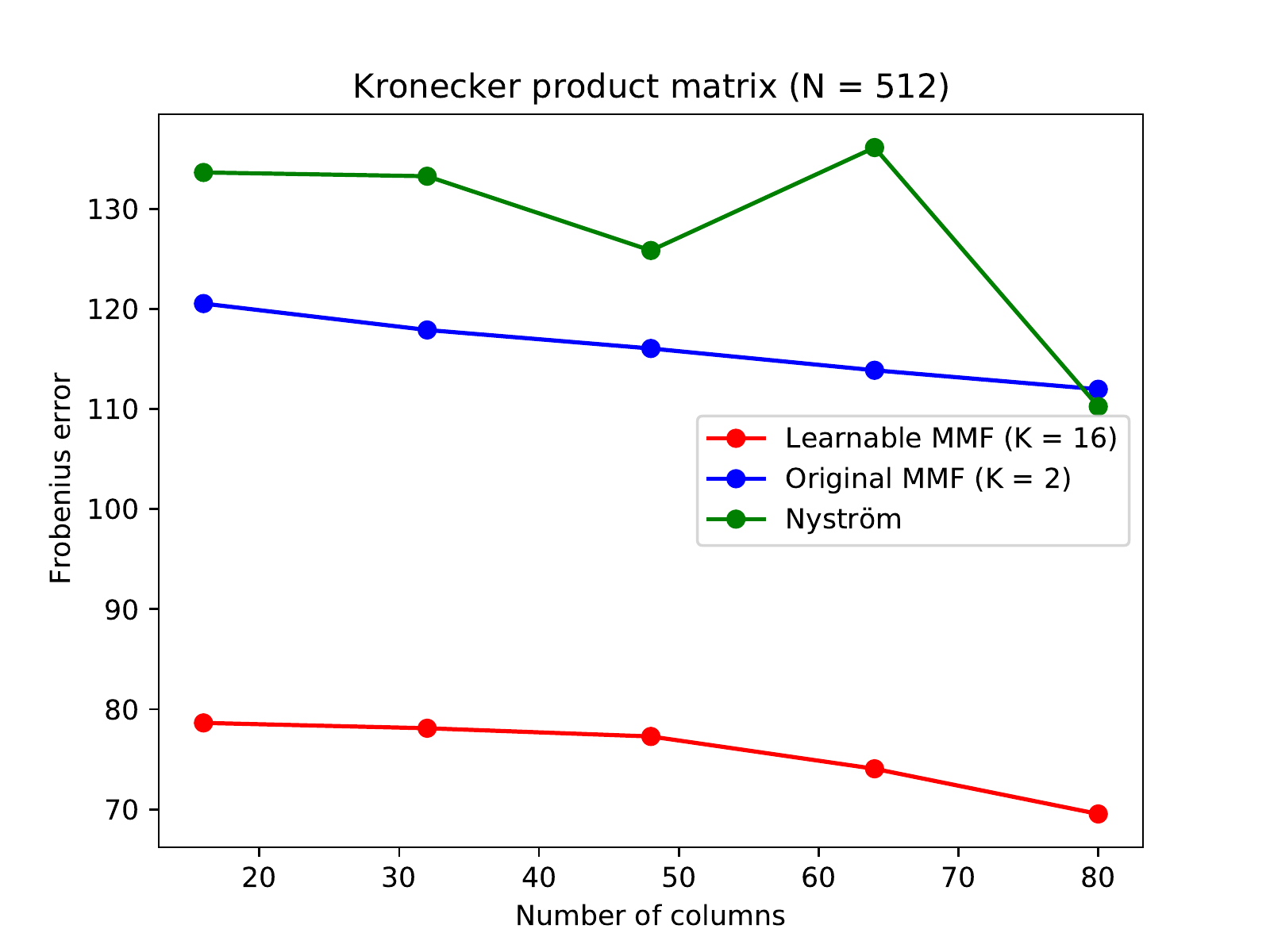}
\includegraphics[width=0.30\textwidth]{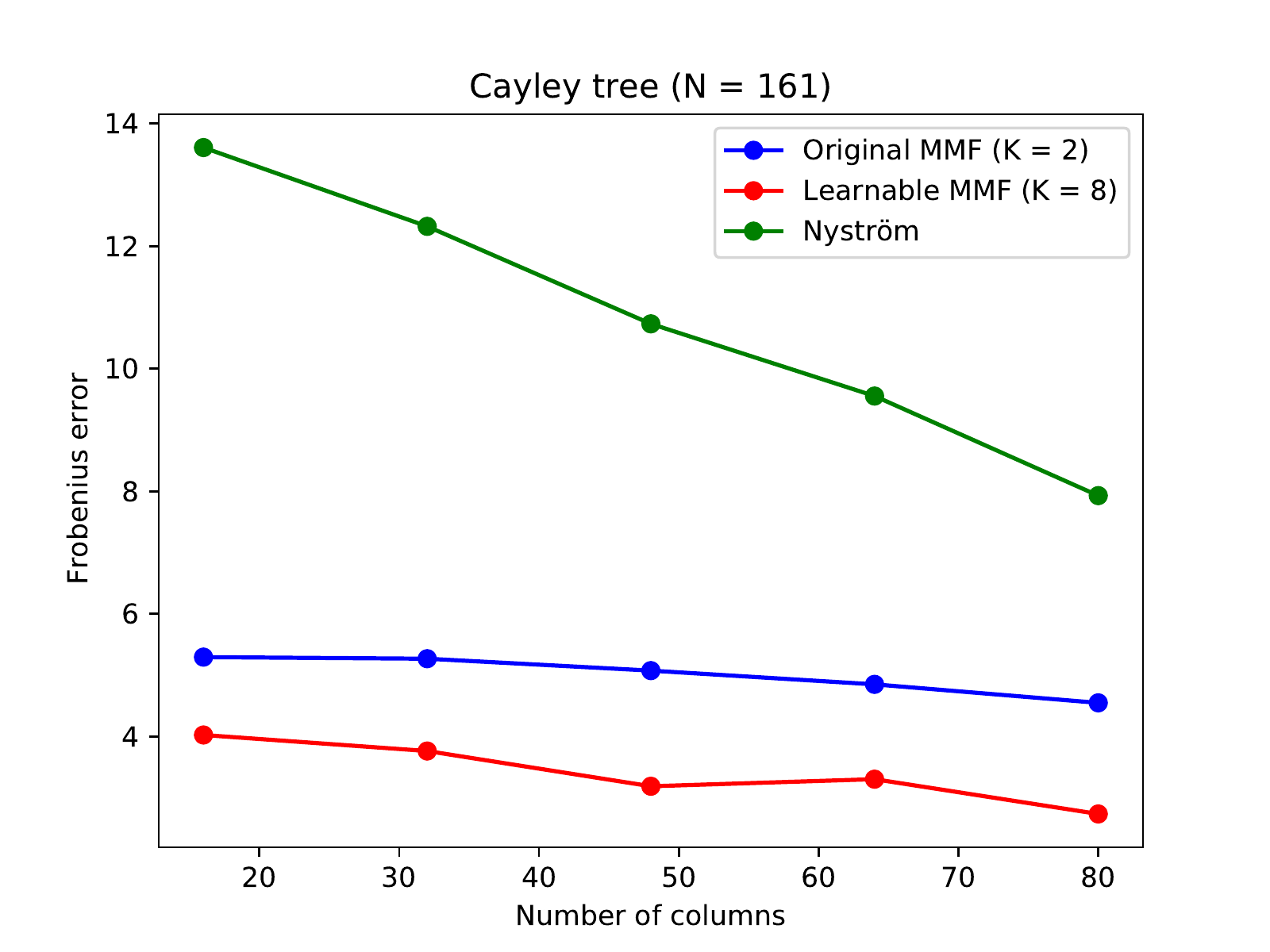}
\end{center}
\vspace{-10pt}
\caption{\label{fig:matrix} Matrix factorization for the Karate network (left), Kronecker matrix (middle), and Cayley tree (right). 
Our learnable MMF consistently outperforms the classic greed methods.}
\end{figure*}

%% file: table_node.tex
\begin{table}
\setlength\tabcolsep{4.7pt}
\caption{\label{tbl:node-classification} Node classification on citation graphs. Baseline results are taken from \citep{xu2018graph}.}
\begin{center}
\begin{tabular}{||l | c | c ||}
\hline
\textbf{Method} & \textbf{Cora} & \textbf{Citeseer} \\
\hline
MLP & 55.1\% & 46.5\% \\
ManiReg \citep{JMLR:v7:belkin06a} & 59.5\% & 60.1\% \\
SemiEmb \citep{10.1145/1390156.1390303} & 59.0\% & 59.6\% \\
LP \citep{10.5555/3041838.3041953} & 68.0\% & 45.3\% \\
DeepWalk \citep{10.1145/2623330.2623732} & 67.2\% & 43.2\% \\
ICA \citep{Getoor2005} & 75.1\% & 69.1\% \\
Planetoid \citep{10.5555/3045390.3045396} & 75.7\% & 64.7\% \\
\hline
Spectral CNN \citep{ae482107de73461787258f805cf8f4ed} & 73.3\% & 58.9\% \\
ChebyNet \citep{10.5555/3157382.3157527} & 81.2\% & 69.8\% \\
GCN \citep{Kipf:2016tc} & 81.5\% & 70.3\% \\
MoNet \citep{MoNet} & 81.7\% & N/A \\
GWNN \citep{xu2018graph} & 82.8\% & 71.7\% \\
\hline
\textbf{MMF}$_1$ & \textbf{84.35\%} & 68.07\% \\
\textbf{MMF}$_2$ & \textbf{84.55\%} & \textbf{72.76\%} \\
\textbf{MMF}$_3$ & \textbf{87.59\%} & \textbf{72.90\%} \\
\hline
\end{tabular}
\end{center}
\end{table}

%% file: table_graph.tex
\begin{table*}[h]
\small
\caption{\label{tbl:graph-classification} Graph classification. Baseline results are taken from \citep{maron2018invariant}.}
\begin{center}
\small
\begin{tabular}{||l | c | c | c | c | c ||}
\hline
\textbf{Method} & \textbf{MUTAG} & \textbf{PTC} & \textbf{PROTEINS} & \textbf{NCI1} \\
\hline
\hline
DGCNN \citep{Zhang2018AnED} & 85.83 $\pm$ 1.7 & 58.59 $\pm$ 2.5 & 75.54 $\pm$ 0.9 & 74.44 $\pm$ 0.5 \\
\hline
PSCN \citep{DBLP:journals/corr/NiepertAK16} & 88.95 $\pm$ 4.4 & 62.29 $\pm$ 5.7 & 75 $\pm$ 2.5 & 76.34 $\pm$ 1.7 \\
\hline
DCNN \citep{10.5555/3157096.3157320} & N/A & N/A & 61.29 $\pm$ 1.6 & 56.61 $\pm$ 1.0 \\
\hline
CCN \citep{CompNetsArxiv18} & \textbf{91.64 $\pm$ 7.2} & \textbf{70.62 $\pm$ 7.0} & N/A & 76.27 $\pm$ 4.1 \\
\hline
GK \citep{pmlr-v5-shervashidze09a} & 81.39 $\pm$ 1.7 & 55.65 $\pm$ 0.5 & 71.39 $\pm$ 0.3 & 62.49 $\pm$ 0.3 \\
\hline
RW \citep{10.5555/1756006.1859891} & 79.17 $\pm$ 2.1 & 55.91 $\pm$ 0.3 & 59.57 $\pm$ 0.1 & N/A \\
\hline
PK \citep{PK2015} & 76 $\pm$ 2.7 & 59.5 $\pm$ 2.4 & 73.68 $\pm$ 0.7 & 82.54 $\pm$ 0.5 \\
\hline
WL \citep{JMLR:v12:shervashidze11a} & 84.11 $\pm$ 1.9 & 57.97 $\pm$ 2.5 & 74.68 $\pm$ 0.5 & \textbf{84.46 $\pm$ 0.5} \\
\hline
IEGN \citep{maron2018invariant} & 84.61 $\pm$ 10 & 59.47 $\pm$ 7.3 & 75.19 $\pm$ 4.3 & 73.71 $\pm$ 2.6 \\
\hline
\hline
\textbf{MMF} & 86.31 $\pm$ 9.47 & 67.99 $\pm$ 8.55 & \textbf{78.72 $\pm$ 2.53} & 71.04 $\pm$ 1.53 \\
\hline
\end{tabular}
\end{center}
\vspace{-15pt}
\end{table*}

%% file: figure_wavelets.tex
\begin{figure*}
\vspace{-2pt}
\begin{center}
\begin{tabular}{ccc}
\includegraphics[width=0.30\textwidth]{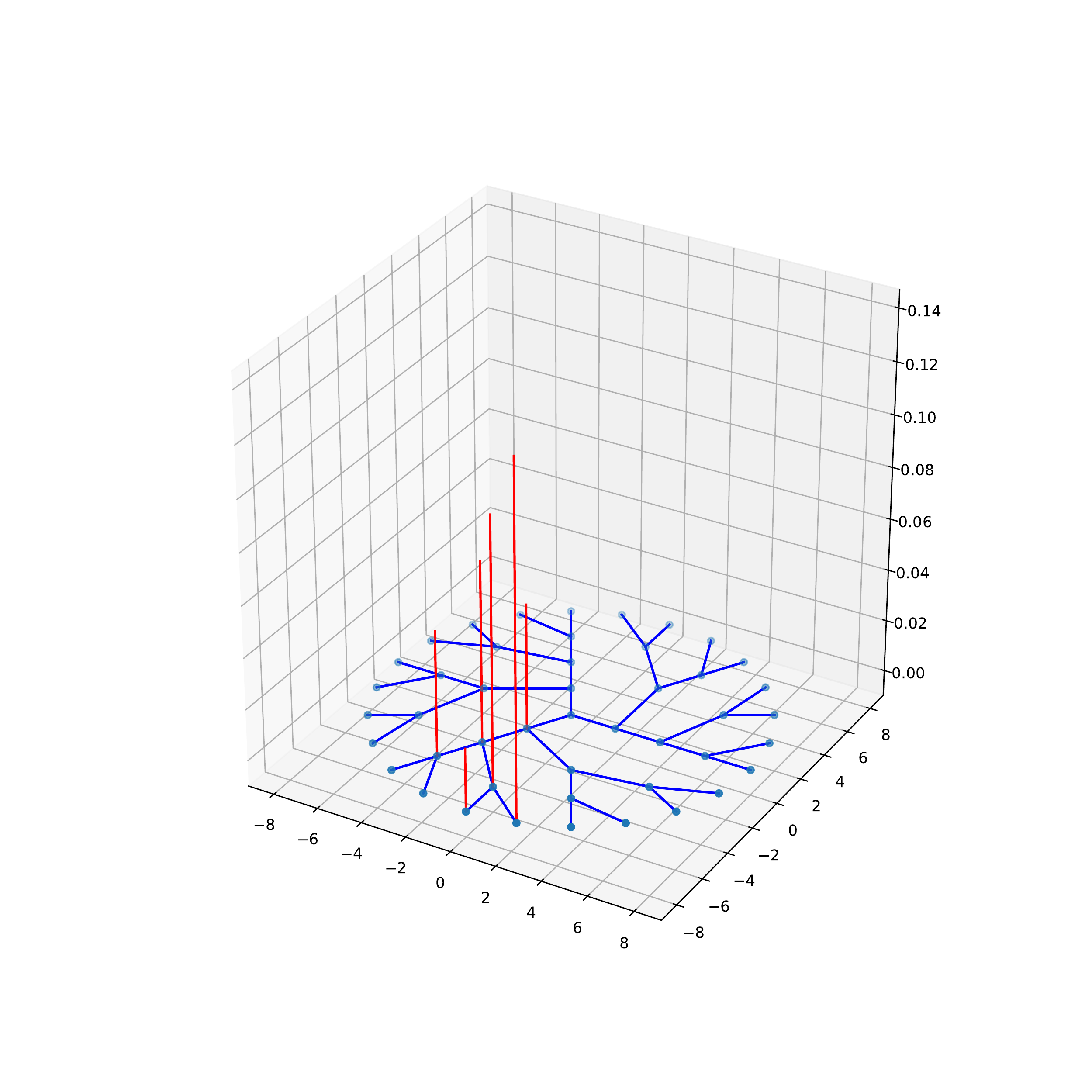}
\vspace{-15pt}
&
\includegraphics[width=0.30\textwidth]{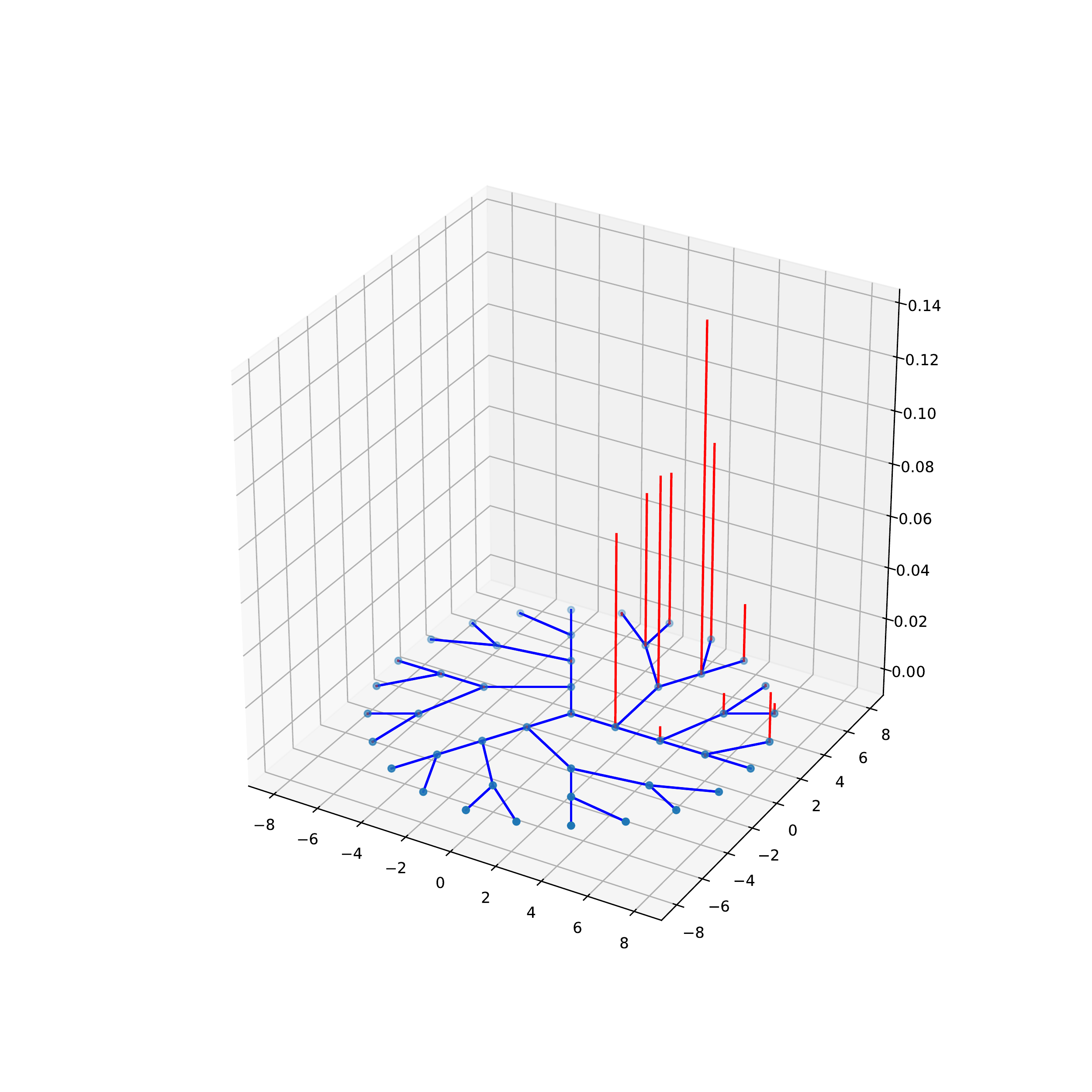}
\vspace{-2pt}
&
\includegraphics[width=0.30\textwidth]{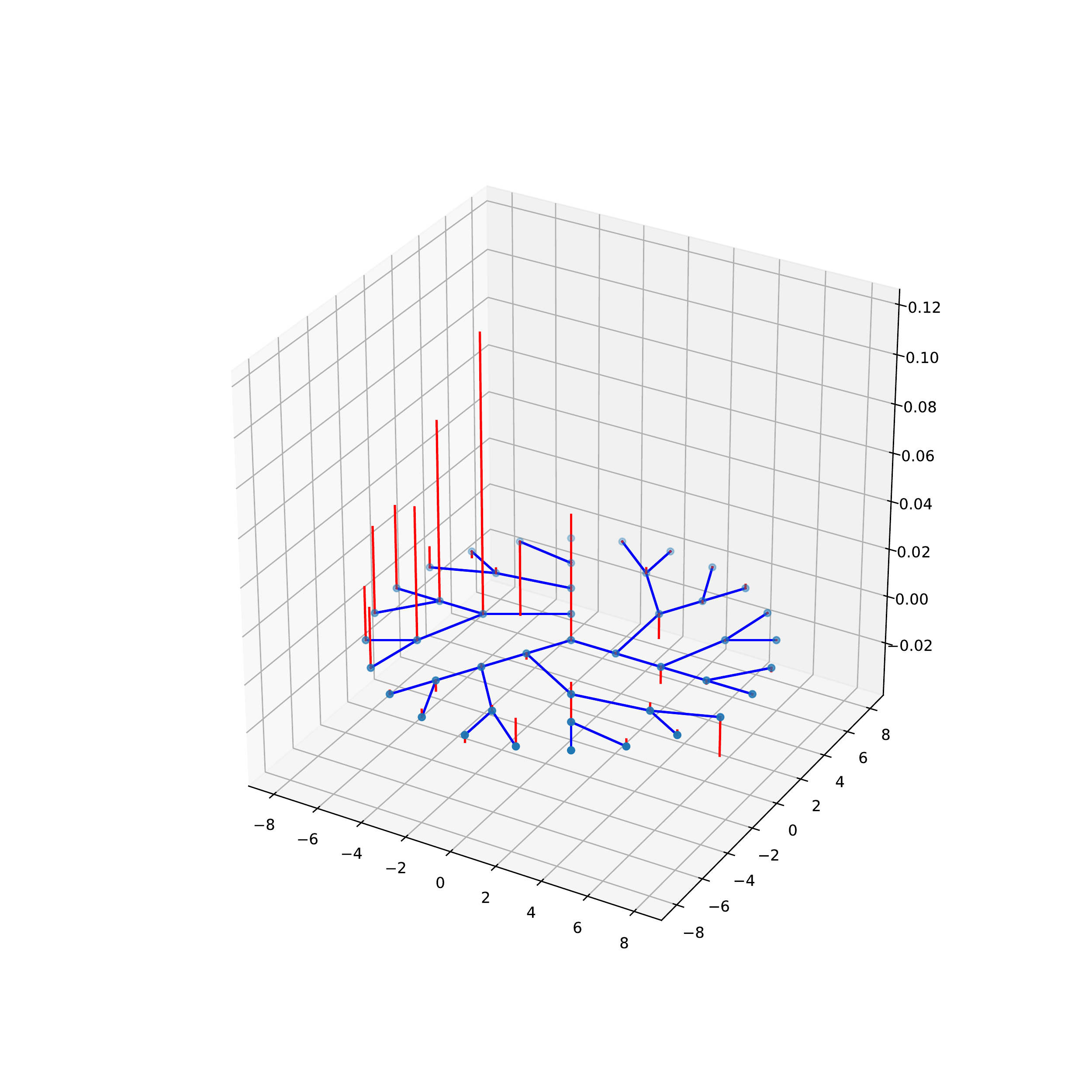}
\vspace{-2pt}
\\
$\ell = 1$
&
$\ell = 20$
&
$\ell = 39$
\end{tabular}
\vspace{-13pt}
\end{center}
\caption{\label{fig:wavelets-visual} Visualization of some of the wavelets on the Cayley tree of 46 vertices. 
The low index wavelets (low $\ell$) are highly localized, whereas the high index ones are smoother and spread out over large
parts of the graph.}
\vspace{-10pt}
\end{figure*}

%% file: Software.tex
\section{Software}

We implemented our learning algorithm for MMF and the wavelet networks by PyTorch deep learning framework \citep{NEURIPS2019_bdbca288}. We released our implementation at 
\begin{center}
\url{https://github.com/risilab/Learnable_MMF/}.
\end{center}

%% file: Conclusion.tex
\section{Conclusions} \label{sec:Conclusion}

In this paper we introduced a general algorithm based on reinforcement learning and Stiefel manifold 
optimization to optimize Multiresolution Matrix Factorization (MMF). 
We find that the resulting learnable MMF consistently outperforms the existing greedy and heuristic 
MMF algorithms in factorizing and approximating hierarchical matrices. 
Based on the wavelet basis returned from our learning algorithm, we define a corresponding 
notion of spectral convolution 
and construct a wavelet neural network for graph learning problems. 
Thanks to the sparsity of the MMF wavelets, the wavelet network can be 
efficiently implemented with sparse matrix multiplication. 
We find that this combination of learnable MMF factorization and spectral wavelet network yields 
state of the art results on standard 
node classification and molecular graph classification.

%% file: Notation.tex
\section{Notation} \label{sec:Notation}

We define $[n] = \{1, 2, \dots, n\}$ as the set of the first $n$ natural numbers. We denote $\mI_n$ as the $n$ dimensional identity matrix. The group of $n$ dimensional orthogonal matrices is $\sS\sO(n)$. $\sA \cupdot \sB$ will denote the disjoint union of two sets $\sA$ and $\sB$, therefore $\sA_1 \cupdot \sA_2 \cupdot \dots \cupdot \sA_k = \sS$ is a partition of $\sS$. \\ \\
Given a matrix $\mA \in \mathbb{R}^{n \times n}$ and two sequences of indices $\vi = (i_1, \dots, i_k) \in [n]^k$ and $\vj = (j_1, \dots, j_k) \in [n]^k$ assuming that $i_1 < i_2 < \dots < i_k$ and $j_1 < j_2 < \dots < j_k$, $\mA_{\vi, \vj}$ will be the $k \times k$ matrix with entries $[\mA_{\vi, \vj}]_{x, y} = \mA_{i_x, j_y}$. Furthermore, $\mA_{i, :}$ and $\mA_{:, j}$ denote the $i$-th row and the $j$-th column of $\mA$, respectively. Given $\mA_1 \in \R^{n_1 \times m_1}$ and $\mA_2 \in \R^{n_2 \times m_2}$, $\mA_1 \oplus \mA_2$ is the $(n_1 + n_2) \times (m_1 + m_2)$ dimensional matrix with entries
$$[\mA_1 \oplus \mA_2]_{i, j} = 
\begin{cases} 
[\mA_1]_{i, j} & \text{if} \ \ i \leq n_1 \ \ \text{and} \ \ j \leq m_1 \\
[\mA_2]_{i - n_1, j - m_1} & \text{if} \ \ i > n_1 \ \ \text{and} \ \ j > m_1 \\
0 & \text{otherwise.}
\end{cases}$$
A matrix $\mA$ is said to be block diagonal if it is of the form
\begin{equation}
\mA = \mA_1 \oplus \mA_2 \oplus \dots \oplus \mA_p
\label{eq:block-matrix}
\end{equation}
for some sequence of smaller matrices $\mA_1, \dots, \mA_p$. For the generalized block diagonal matrix, we remove the restriction that each block in (\ref{eq:block-matrix}) must involve a contiguous set of indices, and introduce the notation
$$\mA = \oplus_{(i_1^1, \dots, i_{k_1}^1)} \mA_1 \oplus_{(i_1^2, \dots, i_{k_2}^2)} \mA_2 \dots \oplus_{(i_1^p, \dots, i_{k_p}^p)} \mA_p$$
in which
$$
\mA_{a, b} =
\begin{cases}
[\mA_u]_{q, r} & \text{if} \ \ i_q^u = a \ \ \text{and} \ \ i^u_r = b \ \ \text{for some} \ \ u, q, r, \\
0 & \text{otherwise.}
\end{cases}
$$
The Kronecker tensor product $\mA_1 \otimes \mA_2$ is an $n_1n_2 \times m_1m_2$ matrix whose elements are
$$[\mA_1 \otimes \mA_2]_{(i_1 - 1)n_2 + i_2, (j_1 - 1)m_2 + j_2} = [\mA_1]_{i_1, j_1} \cdot [\mA_2]_{i_2, j_2},$$
with the obvious generalization to $p$-fold products $\mA_1 \otimes \mA_2 \otimes \dots \otimes \mA_p$. We denote $\mA^{\otimes p}$ as the $p$-fold product $\mA \otimes \mA \otimes \dots \otimes \mA$. \\ \\
A matrix $\mA \in \R^{n \times n}$ is called skew-symmetric (or anti-symmetric) if $\mA^T = -\mA$. The Euclidean inner product between two matrices $\mA \in \R^{m \times n}$ and $\mB \in \R^{m \times n}$ is defined as 
$$\langle \mA, \mB \rangle = \sum_{j, k} \mA_{j, k} \mB_{j, k} = \text{trace}(\mA^T \mB).$$
The Frobenius norm of $\mA$ is defined as $||\mA||_F = \sqrt{\sum_{i, j} \mA_{i, j}^2}$.

%% file: Multiresolution_Matrix_Factorization.tex
\section{Multiresolution Matrix Factorization} \label{sec:MMF}

\input{Background}
\input{Formal_Definition}
\input{subspaces_figure}
\input{MMF_figures}
\input{Optimization}

%% file: Background.tex
\subsection{Background} \label{sec:Background}

Most commonly used matrix factorization algorithms, such as principal component analysis (PCA), singular value decomposition (SVD), or non-negative matrix factorization (NMF) are inherently single-level algorithms. Saying that a symmetric matrix $\mA \in \R^{n \times n}$ is of rank $r \ll n$ means that it can be expressed in terms of a dictionary of $r$ mutually orthogonal unit vectors $\{u_1, u_2, \dots, u_r\}$ in the form
$$\mA = \sum_{i = 1}^r \lambda_i u_i u_i^T,$$
where $u_1, \dots, u_r$ are the normalized eigenvectors of $A$ and $\lambda_1, \dots, \lambda_r$ are the corresponding eigenvalues. This is the decomposition that PCA finds, and it corresponds to factorizing $\mA$ in the form
\begin{equation}
\mA = \mU^T \mH \mU,
\label{eq:eigen}
\end{equation}
where $\mU$ is an orthogonal matrix and $\mH$ is a diagonal matrix with the eigenvalues of $\mA$ on its diagonal. The drawback of PCA is that eigenvectors are almost always dense, while matrices occuring in learning problems, especially those related to graphs, often have strong locality properties, in the sense that they are more closely couple certain clusters of nearby coordinates than those farther apart with respect to the underlying topology. In such cases, modeling $A$ in terms of a basis of global eigenfunctions is both computationally wasteful and conceptually unreasonable: a localized dictionary would be more appropriate. In contrast to PCA, \citep{DBLP:conf/icml/KondorTG14} proposed \textit{Multiresolution Matrix Factorization}, or MMF for short, to construct a sparse hierarchical system of $L$-level dictionaries. The corresponding matrix factorization is of the form
$$\mA = \mU_1^T \mU_2^T \dots \mU_L^T \mH \mU_L \dots \mU_2 \mU_1,$$
where $\mH$ is close to diagonal and $\mU_1, \dots, \mU_L$ are sparse orthogonal matrices with the following constraints:
\begin{enumerate}
\item Each $\mU_\ell$ is $k$-point rotation for some small $k$, meaning that it only rotates $k$ coordinates at a time. Formally, Def.~\ref{def:rotation-matrix} defines and Fig.~\ref{fig:rotation-matrix} shows an example of the $k$-point rotation matrix. 
\item There is a nested sequence of sets $\sS_L \subseteq \cdots \subseteq \sS_1 \subseteq \sS_0 = [n]$ such that the coordinates rotated by $\mU_\ell$ are a subset of $\sS_\ell$.
\item $\mH$ is an $\sS_L$-core-diagonal matrix that is formally defined in Def.~\ref{def:core-diagonal}.
\end{enumerate}
\begin{definition} \label{def:rotation-matrix}
We say that $\mU \in \R^{n \times n}$ is an \textbf{elementary rotation of order $k$} (also called as a $k$-point rotation) if it is an orthogonal matrix of the form
$$\mU = \mI_{n - k} \oplus_{(i_1, \cdots, i_k)} \mO$$
for some $\sI = \{i_1, \cdots, i_k\} \subseteq [n]$ and $\mO \in \sS\sO(k)$. We denote the set of all such matrices as $\sS\sO_k(n)$.
\end{definition}
The simplest case are second order rotations, or called Givens rotations, which are of the form
\begin{equation}
\mU = \mI_{n - 2} \oplus_{(i, j)} \mO = 
\begin{pmatrix}
\cdot & & & & \\
& \cos(\theta) & & -\sin(\theta) & \\
& & \cdot & & \\
& \sin(\theta) & & \cos(\theta) & \\
& & & & \cdot \\
\end{pmatrix},
\label{eq:givens}
\end{equation}
where the dots denote the identity that apart from rows/columns $i$ and $j$, and $\mO \in \sS\sO(2)$ is the rotation matrix of some angle $\theta \in [0, 2\pi)$. Indeed, Jacobi's algorithm for diagonalizing symmetric matrices \citep{Jacobi+1846+51+94} is a special case of MMF factorization over Givens rotations. \\ \\
\begin{definition} \label{def:core-diagonal}
Given a set $\sS \subseteq [n]$, we say that a matrix $\mH \in \R^{n \times n}$ is $\sS$-core-diagonal if $\mH_{i, j} = 0$ unless $i, j \in \sS$ or $i = j$. Equivalently, $\mH$ is $\sS$-core-diagonal if it can be written in the form $\mH = \mD \oplus_{\sS} \overline{\mH}$, for some $\overline{H} \in \R^{|\sS| \times |\sS|}$ and $\mD$ is diagonal. We denote the set of all $\sS$-core-diagonal symmetric matrices of dimension $n$ as $\sH^{\sS}_n$.
\end{definition}
In general, finding the best MMF factorization to a symmetric matrix $\mA$ requires solving
$$
\min_{\substack{\sS_L \subseteq \cdots \subseteq \sS_1 \subseteq \sS_0 = [n] \\ \mH \in \sH^{\sS_L}_n; \mU_1, \dots, \mU_L \in \sO}} || \mA - \mU_1^T \dots \mU_L^T \mH \mU_L \dots \mU_1 ||.
$$

%% file: Formal_Definition.tex
\subsection{Multiresolution analysis} \label{sec:mra}

We formally define MMF in Defs.~\ref{def:mmf} and \ref{def:factorizable}. Furthermore, \citep{DBLP:conf/icml/KondorTG14} has shown that MMF mirrors the classical theory of multiresolution analysis (MRA) on the real line \citep{192463} to discrete spaces. The functional analytic view of wavelets is provided by MRA, which, similarly to Fourier analysis, is a way of filtering some function space into a sequence of subspaces
\begin{equation}
\dots \subset \sV_{-1} \subset \sV_0 \subset \sV_1 \subset \sV_2 \subset \dots
\label{eq:subspace-sequence}
\end{equation}
\begin{definition} \label{def:mmf}
Given an appropriate subset $\sO$ of the group $\sS\sO(n)$ of $n$-dimensional rotation matrices, a depth parameter $L \in \mathbb{N}$, and a sequence of integers $n = d_0 \ge d_1 \ge d_2 \ge \dots \ge d_L \ge 1$, a \textbf{Multiresolution Matrix Factorization (MMF)} of a symmetric matrix $\mA \in \R^{n \times n}$ over $\sO$ is a factorization of the form
\begin{equation} \label{eq:mmf}
\mA = \mU_1^T \mU_2^T \dots \mU_L^T \mH \mU_L \dots \mU_2 \mU_1,
\end{equation}
where each $\mU_\ell \in \sO$ satisfies $[\mU_\ell]_{[n] \setminus \sS_{\ell - 1}, [n] \setminus \sS_{\ell - 1}} = \mI_{n - d_\ell}$ for some nested sequence of sets $\sS_L \subseteq \cdots \subseteq \sS_1 \subseteq \sS_0 = [n]$ with $|\sS_\ell| = d_\ell$, and $\mH \in \sH^{\sS_L}_n$ is an $\sS_L$-core-diagonal matrix.
\end{definition}
\begin{definition} \label{def:factorizable}
We say that a symmetric matrix $\mA \in \R^{n \times n}$ is \textbf{fully multiresolution factorizable} over $\sO \subset \sS\sO(n)$ with $(d_1, \dots, d_L)$ if it has a decomposition of the form described in Def.~\ref{def:mmf}.
\end{definition}
\noindent
However, it is best to conceptualize (\ref{eq:subspace-sequence}) as an iterative process of splitting each $\sV_\ell$ into the orthogonal sum $\sV_\ell = \sV_{\ell + 1} \oplus \sW_{\ell + 1}$ of a smoother part $\sV_{\ell + 1}$, called the \textit{approximation space}; and a rougher part $\sW_{\ell + 1}$, called the \textit{detail space} (see Fig.~\ref{fig:subspaces}). Each $\sV_\ell$ has an orthonormal basis $\Phi_\ell \triangleq \{\phi_m^\ell\}_m$ in which each $\phi$ is called a \textit{father} wavelet. Each complementary space $\sW_\ell$ is also spanned by an orthonormal basis $\Psi_\ell \triangleq \{\psi_m^\ell\}_m$ in which each $\psi$ is called a \textit{mother} wavelet. In MMF, each individual rotation $\mU_\ell: \sV_{\ell - 1} \rightarrow \sV_\ell \oplus \sW_\ell$ is a sparse basis transform that expresses $\Phi_\ell \cup \Psi_\ell$ in the previous basis $\Phi_{\ell - 1}$ such that:
$$\phi_m^\ell = \sum_{i = 1}^{\text{dim}(\sV_{\ell - 1})} [\mU_\ell]_{m, i} \phi_i^{\ell - 1},$$
$$\psi_m^\ell = \sum_{i = 1}^{\text{dim}(\sV_{\ell - 1})} [\mU_\ell]_{m + \text{dim}(\sV_{\ell - 1}), i} \phi_i^{\ell - 1},$$
in which $\Phi_0$ is the standard basis, i.e. $\phi_m^0 = e_m$; and $\text{dim}(\sV_\ell) = d_\ell = |\sS_\ell|$. In the $\Phi_1 \cup \Psi_1$ basis, $\mA$ compresses into $\mA_1 = \mU_1\mA\mU_1^T$. In the $\Phi_2 \cup \Psi_2 \cup \Psi_1$ basis, it becomes $\mA_2 = \mU_2\mU_1\mA\mU_1^T\mU_2^T$, and so on. Finally, in the $\Phi_L \cup \Psi_L \cup \dots \cup \Psi_1$ basis, it takes on the form $\mA_L = \mH = \mU_L \dots \mU_2\mU_1 \mA \mU_1^T\mU_2^T \dots \mU_L^T$ that consists of four distinct blocks (supposingly that we permute the rows/columns accordingly):
$$\mH = \begin{pmatrix} \mH_{\Phi, \Phi} & \mH_{\Phi, \Psi} \\ \mH_{\Psi, \Phi} & \mH_{\Psi, \Psi} \end{pmatrix},$$
where $\mH_{\Phi, \Phi} \in \R^{\text{dim}(\sV_L) \times \text{dim}(\sV_L)}$ is effectively $\mA$ compressed to $\sV_L$, $\mH_{\Phi, \Psi} = \mH_{\Psi, \Phi}^T = 0$ and $\mH_{\Psi, \Psi}$ is diagonal. MMF approximates $\mA$ in the form
$$\mA \approx \sum_{i, j = 1}^{d_L} h_{i, j} \phi_i^L {\phi_j^L}^T + \sum_{\ell = 1}^L \sum_{m = 1}^{d_\ell} c_m^\ell \psi_m^\ell {\psi_m^\ell}^T,$$
where $h_{i, j}$ coefficients are the entries of the $\mH_{\Phi, \Phi}$ block, and $c_m^\ell = \langle \psi_m^\ell, \mA \psi_m^\ell \rangle$ wavelet frequencies are the diagonal elements of the $\mH_{\Psi, \Psi}$ block. \\ \\
In particular, the dictionary vectors corresponding to certain rows of $\mU_1$ are interpreted as level one wavelets, the dictionary vectors corresponding to certain rows of $\mU_2\mU_1$ are interpreted as level two wavelets, and so on (see Section \ref{sec:mra}). One thing that is immediately clear is that whereas Eq.~(\ref{eq:eigen}) diagonalizes $\mA$ in a single step, multiresolution analysis will involve a sequence of basis transforms $\mU_1, \mU_2, \dots, \mU_L$, transforming $\mA$ step by step as
\begin{equation}
\mA \rightarrow \mU_1\mA\mU_1^T \rightarrow \mU_2\mU_1\mA\mU_1^T\mU_2^T \rightarrow \dots \rightarrow \mU_L \dots \mU_2\mU_1\mA\mU_1^T\mU_2^T \dots \mU_L^T,
\label{eq:mmf-transform}
\end{equation}
so the corresponding matrix factorization must be a multilevel factorization
\begin{equation}
\mA \approx \mU_1^T \mU_2^T \dots \mU_\ell^T \mH \mU_\ell \dots \mU_2 \mU_1.
\label{eq:mmf-factorization}
\end{equation}
Fig.~\ref{fig:mmf-transform} depicts the multiresolution transform of MMF as in Eq.~(\ref{eq:mmf-transform}). Fig.~\ref{fig:mmf-factorization} illustrates the corresponding factorization as in Eq.~(\ref{eq:mmf-factorization}).

%% file: subspaces_figure.tex
\begin{figure}[t]
$$
\xymatrix{
L_2(\sX) \ar[r] & \cdots \ar[r] & \sV_0 \ar[r] \ar[dr] & \sV_1 \ar[r] \ar[dr] & \sV_2 \ar[r] \ar[dr] & \cdots \\
& & & \sW_1 & \sW_2 & \sW_3
}
$$
\caption{\label{fig:subspaces}
Multiresolution analysis splits each function space $\sV_0, \sV_1, \dots$ into the direct sum of a smoother part $\sV_{\ell + 1}$ and a rougher part $\sW_{\ell + 1}$.
}
\end{figure}
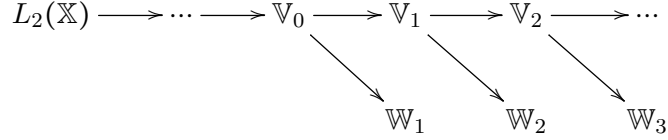

%% file: MMF_figures.tex
\newcommand{\tikmxA}[1]{
\bigg(\,\begin{tikzpicture}[baseline=-17, scale=0.06]
\filldraw[gray] (0,0) rectangle +(#1,-#1); \end{tikzpicture}\,\bigg)}

\newcommand{\tikmxB}[2]{
\bigg(\,\begin{tikzpicture}[baseline=-17, scale=0.06]
\filldraw[gray] (0,0) rectangle +(#1,-#1);
\foreach \i in {#1,...,#2}{
\filldraw[gray] (\i,-\i) rectangle +(1,-1);}
\end{tikzpicture}\,\bigg)}

\newcommand{\tikmxC}[2]{
\bigg(\,\begin{tikzpicture}[baseline=-17, scale=0.06]
\draw (0,0) rectangle +(17,-17);\draw (10,-9) node {#1}; \end{tikzpicture} \,\bigg)}

\newcommand{\tikmxD}[3]{
\bigg(\,\begin{tikzpicture}[baseline=-17, scale=0.06]
\filldraw[gray] (#1,-#1) rectangle +(#2,-#2);
\foreach \i in {0,...,#3}{
\filldraw[gray] (\i,-\i) rectangle +(1,-1);}
\end{tikzpicture}\,\bigg)}

\begin{figure}[t]
\[
\mI_{n - k} \oplus_{(i_1, .., i_k)} \mO = 
\,\Pi\, \underset{\displaystyle \mU}{\tikmxD{7}{4}{17}} \,\Pi^\top
\]\\ \vspace{-30pt}\mbox{}
\caption{\label{fig:rotation-matrix}
A rotation matrix of order $k$. The purpose of permutation matrix $\Pi$ is solely to ensure that the blocks of the matrices appear contiguous in the figure. In this case, $n = 17$ and $k = 4$.
}
\end{figure}

\begin{figure}[t]
\[
\,\Pi\, \underset{\displaystyle \mA}{\tikmxA{17}} \,\Pi^\top \xrightarrow{~~U_1~~} 
\underset{\displaystyle \mA_1 = \mU_1 \mA \mU_1^T}{\tikmxB{14}{17}}\xrightarrow{~~U_2~~}
\underset{\displaystyle \mA_2 = \mU_2 \mA_1 \mU_2^T}{\tikmxB{10}{17}}\xrightarrow{~~~~}
\ldots\ldots\xrightarrow{~~~~}
\underset{\displaystyle \mA_L = \mH}{\tikmxB{7}{17}}
\]\\ \vspace{-30pt}\mbox{}
\caption{\label{fig:mmf-transform}
MMF can be thought of as a process of successively compressing $\mA$ to size $d_1 \times d_1$, $d_2 \times d_2$, etc. (plus the diagonal entries) down to the final $d_L \times d_L$ core-diagonal matrix $\mH$ (see Def.~\ref{def:mmf}). The role of permutation matrix $\Pi$ is purely for the ease of visualization (as in Fig.~\ref{fig:rotation-matrix}).
}
\end{figure}

\begin{figure}[t]
\[
\,\Pi\, \underset{\displaystyle A}{\tikmxA{17}} \,\Pi^\top 
\approx
\underset{\displaystyle U_1^T}{\tikmxD{7}{4}{17}}
\ldots
\underset{\displaystyle U_L^T}{\tikmxD{14}{4}{17}}
\underset{\displaystyle H}{\tikmxD{0}{7}{17}}
\underset{\displaystyle U_L}{\tikmxD{14}{4}{17}}
\ldots
\underset{\displaystyle U_1}{\tikmxD{7}{4}{17}}
\]\\ \vspace{-30pt}\mbox{}
\caption{\label{fig:mmf-factorization}
Matrix approximation as in Eq.~\ref{eq:mmf}. In this figure, the core block size of each rotation matrix $\mU_\ell$ and $\mH$ are $k \times k = 4 \times 4$ and $d_L \times d_L = 8 \times 8$, respectively. Permutation matrix $\Pi$ is only for visualization (as in Figs.~\ref{fig:rotation-matrix}~\ref{fig:mmf-transform}).
}
\end{figure}

%% file: Optimization.tex
\subsection{Optimization by heuristics} \label{sec:Optimization}

Heuristically, factorizing $\mA$ can be approximated by an iterative process that starts by setting $\mA_0 = \mA$ and $\sS_1 = [n]$, and then executes the following steps for each resolution level $\ell \in \{1, \dots, L\}$:
\begin{enumerate}
\item Given $\mA_{\ell - 1}$, select $k$ indices $\sI_\ell = \{i_1, \dots, i_k\} \subset \sS_{\ell - 1}$ of rows/columns of the active submatrix $[\mA_{\ell - 1}]_{\sS_{\ell - 1}, \sS_{\ell - 1}}$ that are highly correlated with each other.
\item Find the corresponding $k$-point rotation $\mU_\ell$ to $\sI_\ell$, and compute $\mA_\ell = \mU_\ell \mA_{\ell - 1} \mU_\ell^T$ that brings the submatrix $[\mA_{\ell - 1}]_{\sI_\ell, \sI_\ell}$ close to diagonal. In the last level, we set $\mH = \mA_L$ (see Fig.~\ref{fig:mmf-transform}).
\item Determine the set of coordinates $\sT_\ell \subseteq \sS_{\ell - 1}$ that are to be designated wavelets at this level, and eliminate them from the active set by setting $\sS_\ell = \sS_{\ell - 1} \setminus \sT_\ell$.
\end{enumerate}

%% file: Stiefel_Manifold_Optimization_2.tex
\section{Stiefel Manifold Optimization} \label{sec:proof}

In order to solve the MMF optimization problem, we consider the following generic optimization with orthogonality constraints:
\begin{equation}
\min_{\mX \in \R^{n \times p}} \mathcal{F}(\mX), \ \ \text{s.t.} \ \ \mX^T \mX = \mI_p,
\label{eq:opt-prob}
\end{equation}
We identify tangent vectors to the manifold with $n \times p$ matrices. We denote the tangent space at $\mX$ as $\mathcal{T}_{\mX} \mathcal{V}_p(\R^n)$. Lemma \ref{lemma:tangent} characterizes vectors in the tangent space.
\begin{lemma}
Any $\mZ \in \mathcal{T}_{\mX} \mathcal{V}_p(\R^n)$, then $\mZ$ (as an element of $\R^{n \times p}$) satisfies
$$\mZ^T \mX + \mX^T \mZ = 0,$$
where $\mZ^T \mX$ is a skew-symmetric $p \times p$ matrix.
\label{lemma:tangent}
\end{lemma}
\begin{proof}
Let $\mY(t)$ be a curve in $\mathcal{V}_p(\R^n)$ that starts from $\mX$. We have:
\begin{equation}
\mY^T(t)\mY(t) = \mI_p.
\label{eq:identity}
\end{equation}
We differentiate two sides of Eq.~(\ref{eq:identity}) with respect to $t$:
$$\frac{d}{dt}(\mY^T(t)\mY(t)) = 0$$
that leads to:
$$\bigg(\frac{d\mY}{dt}(0)\bigg)^T \mY(0) + \mY(0)^T \frac{d\mY}{dt}(0) = 0$$
at $t = 0$. Recall that by definition, $\mY(0) = \mX$ and $\frac{d\mY}{dt}(0)$ is any element of the tangent space at $\mX$. Therefore, we arrive at $\mZ^T\mX + \mX^T\mZ = 0$.
\end{proof}
Suppose that $\mathcal{F}$ is a differentiable function. The gradient of $\mathcal{F}$ with respect to $\mX$ is denoted by $\mG \triangleq \mathcal{D}\mathcal{F}_{\mX} \triangleq \big(\frac{\partial \mathcal{F}(\mX)}{\partial \mX_{i, j}}\big)$. The derivative of $\mathcal{F}$ at $\mX$ in a direction $\mZ$ is
$$\mathcal{D}\mathcal{F}_{\mX}(\mZ) \triangleq \lim_{t \rightarrow 0} \frac{\mathcal{F}(\mX + t\mZ) - \mathcal{F}(\mX)}{t} = \langle \mG, \mZ \rangle$$
Since the matrix $\mX^T \mX$ is symmetric, the Lagrangian multiplier $\Lambda$ corresponding to $\mX^T\mX = \mI_p$ is a symmetric matrix. The Lagrangian function of problem (\ref{eq:opt-prob}) is
\begin{equation}
\mathcal{L}(\mX, \mLambda) = \mathcal{F}(\mX) - \frac{1}{2} \text{trace}(\mLambda (\mX^T \mX - \mI_p))
\label{eq:lagrangian}
\end{equation}
\begin{lemma}
Suppose that $\mX$ is a local minimizer of problem (\ref{eq:opt-prob}). Then $\mX$ satisfies the first-order optimality conditions $\mathcal{D}_{\mX}\mathcal{L}(\mX, \mLambda) = \mG - \mX\mG^T\mX = 0$ and $\mX^T\mX = \mI_p$ with the associated Lagrangian multiplier $\mLambda = \mG^T\mX$. Define $\displaystyle \nabla \mathcal{F}(\mX) \triangleq \mG - \mX \mG^T \mX$ and $\mA \triangleq \mG \mX^T - \mX \mG^T$. Then $\displaystyle \nabla \mathcal{F} = \mA \mX$. Moreover, $\displaystyle \nabla \mathcal{F} = 0$ if and only if $\mA = 0$.
\label{lemma:first-order-condition}
\end{lemma}
\begin{proof}
Since $\mX \in \mathcal{V}_p(\R^n)$, we have $\mX^T\mX = \mI_p$. We differentiate both sides of the Lagrangian function:
$$\mathcal{D}_\mX \mathcal{L}(\mX, \mLambda) = \mathcal{D}\mathcal{F}(\mX) - \mX \mLambda = 0.$$
Recall that by definition, $\mG \triangleq \mathcal{D}\mathcal{F}(\mX)$, we have
\begin{equation}
\mathcal{D}_\mX \mathcal{L}(\mX, \mLambda) = \mG - \mX \mLambda = 0.
\end{equation}
Multiplying both sides by $\mX^T$, we get $\mX^T\mG - \mX^T\mX \mLambda = 0$ that leads to $\mX^T\mG - \mLambda = 0$ or $\mLambda = \mX^T\mG$. Since the matrix $\mX^T\mX$ is symmetric, the Lagrangian multiplier $\mLambda$ correspoding to $\mX^T\mX = \mI_p$ is a symmetric matrix. Therefore, we obtain $\mLambda = \mLambda^T = \mG^T\mX$ and $\mathcal{D}_\mX \mathcal{L}(\mX, \mLambda) = \mG - \mX\mG^T\mX = 0$. By definition, $\mA \triangleq \mG\mX^T - \mX\mG^T$. We have $\mA\mX = \mG - \mX\mG^T\mX = \nabla \mathcal{F}$. The last statement is trivial.
\end{proof}
Let $\mX \in \mathcal{V}_p(\R^n)$, and $\mW$ be any $n \times n$ skew-symmetric matrix. We consider the following curve that transforms $\mX$ by $\big( \mI + \frac{\tau}{2} \mW \big)^{-1} \big(\mI - \frac{\tau}{2} \mW\big)$:
\begin{equation}
\mY(\tau) = \big( \mI + \frac{\tau}{2} \mW \big)^{-1} \big(\mI - \frac{\tau}{2} \mW\big) \mX.
\label{eq:cayley}
\end{equation}
This is called as the \textit{Cayley transformation}. Its derivative with respect to $\tau$ is
\begin{equation}
\mY'(\tau) = -\bigg(\mI + \frac{\tau}{2} \mW\bigg)^{-1} \mW \bigg( \frac{\mX + \mY(\tau)}{2} \bigg).
\label{eq:y-derivative}
\end{equation}
The curve has the following properties:
\begin{enumerate}
\item It stays in the Stiefel manifold, i.e. $\mY(\tau)^T \mY(\tau) = \mI$.
\item Its tangent vector at $\tau = 0$ is $\mY'(0) = -\mW \mX$. It can be easily derived from Lemma \ref{lemma:tangent} that $\mY'(0)$ is in the tangent space $\mathcal{T}_{\mY(0)} \mathcal{V}_p(\R^n)$. Since $\mY(0) = X$ and $\mW$ is a skew-symmetric matrix, by letting $\mZ = -\mW\mX$, it is trivial that $\mZ^T\mX + \mX^T\mZ = 0$.
\end{enumerate}
\begin{lemma}
If we set $\mW \triangleq \mA \triangleq \mG\mX^T - \mX\mG^T$ (see Lemma ~\ref{lemma:first-order-condition}), then the curve $\mY(\tau)$ (defined in Eq.~(\ref{eq:cayley})) is a decent curve for $\mathcal{F}$ at $\tau = 0$, that is
$$\mathcal{F}'_\tau(\mY(0)) \triangleq \frac{\partial \mathcal{F}(\mY(\tau))}{\partial \tau}\bigg|_{\tau = 0} = -\frac{1}{2} ||\mA||_F^2.$$
\label{lemma:descent-curve}
\end{lemma}
\begin{proof}
By the chain rule, we get 
$$\mathcal{F}'_\tau(\mY(\tau)) = \text{trace}(\mathcal{D}\mathcal{F}(\mY(\tau))^T \mY'(\tau)).$$
At $\tau = 0$, $\mathcal{D}\mathcal{F}(\mY(0)) = \mG$ and $\mY'(0) = -\mA\mX$. Therefore, 
$$\mathcal{F}'_\tau(\mY(0)) = -\text{trace}(\mG^T(\mG\mX^T - \mX\mG^T)\mX) = -\frac{1}{2}\text{trace}(\mA\mA^T) = -\frac{1}{2}||\mA||_F^2.$$
\end{proof}
It is well known that the steepest descent method with a fixed step size may not converge, but the convergence can be guaranteed by choosing the step size wisely: one can choose a step size by minimizing $\mathcal{F}(\mY(\tau))$ along the curve $\mY(\tau)$ with respect to $\tau$ \citep{Wen10}. With the choice of $\mW$ given by Lemma \ref{lemma:descent-curve}, the minimization algorithm using $\mY(\tau)$ is roughly sketched as follows: Start with some initial $\mX^{(0)}$. For $t > 0$, we generate $\mX^{(t + 1)}$ from $\mX^{(t)}$ by a curvilinear search along the curve $\mY(\tau) = \big( \mI + \frac{\tau}{2} \mW \big)^{-1} \big(\mI - \frac{\tau}{2} \mW\big) \mX^{(t)}$ by changing $\tau$. Because finding the global minimizer is computationally infeasible, the search terminates when then Armijo-Wolfe conditions that indicate an approximate minimizer are satisfied. The Armijo-Wolfe conditions require two parameters $0 < \rho_1 < \rho_2 < 1$ \citep{NoceWrig06} \citep{Wen10} \citep{Tagare2011NotesOO}:
\begin{equation}
\mathcal{F}(\mY(\tau)) \leq \mathcal{F}(\mY(0)) + \rho_1\tau\mathcal{F}'_\tau(\mY(0))
\label{eq:condition-1}
\end{equation}
\begin{equation}
\mathcal{F}'_\tau(\mY(\tau)) \ge \rho_2\mathcal{F}'_\tau(\mY(0))
\label{eq:condition-2}
\end{equation}
where $\mathcal{F}'_\tau(\mY(\tau)) = \text{trace}(\mG^T\mY'(\tau))$ while $\mY'(\tau)$ is computed as Eq.~(\ref{eq:y-derivative}) and $\mY'(0) = -\mA\mX$. The gradient descent algorithm on Stiefel manifold to optimize the generic orthogonal-constraint problem (\ref{eq:opt-prob}) with the curvilinear search submodule is described in Algorithm \ref{alg:stiefel}, which is used as a submodule in part of our learning algorithm to solve the MMF in (\ref{eq:mmf-opt}). The algorithm can be trivially extended to solve problems with multiple variables and constraints.

\begin{algorithm}
\caption{Stiefel manifold gradient descent algorithm} \label{alg:stiefel}
\begin{algorithmic}[1]
\State Given $0 < \rho_1 < \rho_2 < 1$ and $\epsilon > 0$.
\State Given an initial point $\mX^{(0)} \in \mathcal{V}_p(\R^n)$.
\State $t \gets 0$
\While{true}
	\State $\mG \gets \big(\frac{\partial \mathcal{F}(\mX^{(t)})}{\partial \mX^{(t)}_{i, j}}\big)$ \Comment{Compute the gradient of $\mathcal{F}$ w.r.t $\mX$ elemense-wise}
	\State $\mA \gets \mG{\mX^{(t)}}^T - \mX^{(t)}\mG^T$ \Comment{See Lemma 2, 3}
	\State Initialize $\tau$ to a non-zero value. \Comment{Curvilinear search for the optimal step size}
	\While{(\ref{eq:condition-1}) and (\ref{eq:condition-2}) are \textbf{not} satisfied} \Comment{Armijo-Wolfe conditions}
		\State $\tau \gets \frac{\tau}{2}$ \Comment{Reduce the step size by half}
	\EndWhile
	\State $\mX^{(t + 1)} \gets \mY(\tau)$ \Comment{Update by the Cayley transformation}
	\If{$||\nabla \mathcal{F}(\mX^{(t + 1)})|| \leq \epsilon$} \Comment{Stopping check. See Lemma 2.}
		\State \textbf{STOP}
	\Else
		\State $t \gets t + 1$
	\EndIf
\EndWhile
\end{algorithmic}
\end{algorithm}

%% file: Reinforcement_Learning_2.tex
\section{Reinforcement Learning} \label{sec:RL_2}

\subsection{Policy gradient methods} \label{sec:GP}

Policy gradient has been a widely used approach to solve reinforcement learning problems that targets at modeling and optimizing the policy directly \citep{10.5555/3312046}. Monte-Carlo policy gradient (REINFORCE) \citep{SCC.Williams1988} \citep{10.1007/BF00992696} \citep{NIPS1999_464d828b} depends on an estimated return by Monte-Carlo methods using episode samples to update the learnable parameters $\theta$ of the policy network $\pi_\theta$. We define the value of state $s$ when we follow a policy $\pi$ as $V^\pi(s) = \mathbb{E}_{a \sim \pi}[g_\ell|s_\ell = s]$. The value of (state, action) pair when we follow a policy $\pi$ is defined similarly as $Q^\pi(s, a) = \mathbb{E}_{a \sim \pi}[g_\ell|s_\ell = s, a_\ell = a]$. The value of the reward objective function depends on the policy and is defined as
\begin{equation}
J(\theta) = \sum_{s \in S} d^\pi(s) V^\pi(s) = \sum_{s \in S} d^\pi(s) \sum_{a \in A} \pi_\theta(a|s) Q^\pi(s, a)
\label{eq:reward-obj}
\end{equation}
where $d^\pi(s) = p(s_L = s|s_0, \pi_\theta)$ is the stationary distribution of Markov chain for $\pi_\theta$. It is important to remark that our MDP process terminates after a finite number of transitions (e.g., $L$), so $d^\pi(s)$ is the probability that we end up at state $s$ when starting from $s_0$ and following policy $\pi_\theta$ for $L$ steps. \citep{NIPS1999_464d828b} has shown that an unbiased estimate of the gradient of (\ref{eq:reward-obj}) can be obtained from experience using an approximate value function. The expectation of the sample gradient is equal to the actual gradient:
\begin{equation}
\nabla_\theta J(\theta) = \mathbb{E}_\pi[Q^\pi(s, a) \nabla_\theta \log \pi_\theta(a|s)] = \mathbb{E}_\pi[g_\ell \nabla_\theta \log \pi_\theta(a_\ell|s_\ell)]
\label{eq:gradient}
\end{equation}
that allows us to update our policy gradient by measuring $g_\ell$ from real sample trajectories. Based on (\ref{eq:gradient}), the update rule for policy parameters is simply as 
$$\theta \leftarrow \theta + \eta \gamma^\ell g_\ell \nabla_\theta \log \pi_\theta(a_\ell|s_\ell) \ \ \ \ \text{for} \ \ \ell = 0, .., L - 1;$$
where $\eta$ is the learning rate, that is used in training our policy network $\pi_\theta$.

\subsection{2-phase process} 

The learning algorithm is expensive due to the Stiefel manifold optimization in line 11 to find the optimal rotations $\mO_\ell$ that are used to compute the rewards $g_\ell$. In practice, we propose a $2$-phase process that is more efficient:
\begin{itemize}
\item \textbf{Phase 1:} Reinforcement learning to find the sequence of indices, but instead of manifold optimization, we just use the closed-form solutions for $\mO_\ell$ as the eigenvectors of $\mA_{\sS_\ell, :} \mA_{\sS_\ell, :}^T$ to estimate the rewards. In all our experiments, we implement the policy network by two graph neural networks, one to select the pivot (wavelet index) and the another one to select $K - 1$ indices, with 4 layers of message passing and hidden dimension of 10. The input node feature for node $v$ (or the $v$-th row) is binary: $1$ if $v \in \sS_\ell$, otherwise $0$. We use $\gamma = 1$ as the discount factor and learning rate $\eta = 10^{-3}$.
\item \textbf{Phase 2:} Given a sequence of indices found by the previous phase, we apply Stiefel manifold optimization to actually find the optimal rotations accordingly.
\end{itemize}

\subsection{Transfer learning} \label{sec:transfer}

Ideally, we want our policy network $\pi_\theta$ to be \textit{universal} in the sense that the same trained policy can be applied to different graphs with little adaptation or without any further training. However, the search space is gigantic with large graphs such as social networks, and the cost of training the policy is computationally expensive. Therefore, we apply the idea of \textit{transfer learning} that is to reuse or transfer information from previously learned tasks (source tasks) into new tasks (target tasks). The source task here is to train our GNN policy on a dataset of small graphs (e.g., possibly synthetic graphs) that are much faster to train on, and the target task is to run the trained policy on the large actual graph. For example, given citation networks with thousands of nodes such as Cora and Citeseer \citep{PSen2008}, we generate the dataset for policy training by partitioning the big graph into many smaller connected clusters. The learning algorithm can be easily modified for such a purpose (e.g., training multiple graphs simultaneously). 